%% file: main.tex
\documentclass[acmsmall,screen,nonacm]{acmart}

\AtBeginDocument{%
  \providecommand\BibTeX{{%
    \normalfont B\kern-0.5em{\scshape i\kern-0.25em b}\kern-0.8em\TeX}}}

\setcopyright{acmlicensed}
\copyrightyear{2018}
\acmYear{2018}
\acmDOI{XXXXXXX.XXXXXXX}

\acmConference[Conference acronym 'XX]{Make sure to enter the correct
  conference title from your rights confirmation emai}{June 03--05,
  2018}{Woodstock, NY}
\acmBooktitle{Woodstock '18: ACM Symposium on Neural Gaze Detection,
 June 03--05, 2018, Woodstock, NY} 
\acmISBN{978-1-4503-XXXX-X/18/06}

\usepackage{adjustbox}
\usepackage{amsmath}
\usepackage{mathtools}
\usepackage{amsthm}
\usepackage{multirow}

\usepackage{xspace}
\usepackage{booktabs}
\usepackage{soul}
\usepackage[linesnumbered,ruled,vlined,algo2e]{algorithm2e}

\usepackage[capitalize,noabbrev]{cleveref}

\usepackage{url}
\usepackage{breakurl}

\theoremstyle{plain}
\newtheorem{theorem}{Theorem}[section]

\theoremstyle{definition}
\newtheorem{definition}[theorem]{Definition}

\theoremstyle{remark}

\input{macro-defs.tex}

\usepackage{wrapfig, lipsum}
\usepackage{amsmath}
\usepackage{bussproofs}
\DeclareMathOperator*{\argmax}{argmax}

\usepackage[]{algorithm2e}
\usepackage{scalerel}
\usepackage{graphicx}
\newcommand{\nback}[1][-.95pt]{
  \mathrel{\raisebox{#1}{$\rotatebox[origin=c]{-315}{\scaleobj{0.55}{-}}$}}
}
\newcommand{\undernegpreccurlyeq}{%
\mathrel{\ooalign{$\preccurlyeq$\cr\kern1.2pt$\nback$}}}

\usepackage[textsize=tiny]{todonotes}
\usepackage{algorithm}
\usepackage{algpseudocode}

\usepackage{subcaption}  % Add this line to include the subcaption package

\begin{document}

\title[Learning Minimal Neural Specifications]{Learning Minimal Neural Specifications}

\author{Chuqin Geng}
\email{chuqin.geng@mail.mcgill.ca}
\affiliation{%
  \institution{McGill University}
  \country{Canada}
}

\author{Zhaoyue Wang}
\email{zhaoyue.wang@mail.mcgill.ca}
\affiliation{%
  \institution{McGill University}
  \country{Canada}
}

\author{Haolin Ye}
\email{haolin.ye@mail.mcgill.ca}
\affiliation{%
  \institution{McGill University}
  \country{Canada}
}

% \author{Saifei Liao}
% \email{saifei.liao@mail.utoronto.ca}
% \affiliation{%
%   \institution{University of Toronto}
%   \country{Canada}
% }

\author{Xujie Si}
\email{six@cs.toronto.edu}
\affiliation{%
  \institution{University of Toronto}
  \country{Canada}
}
\renewcommand{\shortauthors}{Geng et al.}


\begin{abstract}

Formal verification is only as good as the specification of a system, which is also true for neural network verification. Existing specifications follow the paradigm of \textit{data as specification}, where the local neighborhood around a reference data point is considered correct or robust. While these specifications provide a fair testbed for assessing model robustness, they are too \textit{restrictive} for verifying any unseen test data points -- a challenging task with significant real-world implications. Recent work shows great promise through a new paradigm, \textit{neural representation as specification}, which uses neural activation patterns (NAPs) for this purpose. However, it computes the most refined NAPs, which include many redundant neurons. In this paper, we study the following problem: Given a neural network, find a minimal (general) NAP specification that is sufficient for formal verification of the network’s global robustness. Finding the minimal NAP specification not only expands verifiable bounds but also provides insights into which set of neurons contributes more to the model’s robustness. To address this problem, we propose three approaches—conservative, statistical, and optimistic—each offering distinct trade-offs between efficiency and performance. The first two rely on the verification tool to find minimal NAP specifications. The optimistic method efficiently estimates minimal NAPs using adversarial examples, without making calls to the verification tool until the very end. Each of these methods offers distinct strengths and trade-offs in terms of minimality and computational speed, making each approach suitable for scenarios with different priorities. The learnt minimal NAP specification allows us to inspect potential causal links between neurons and the robustness of state-of-the-art neural networks, a task for which existing work fails to scale. Our experimental results suggest that minimal NAP specifications require much smaller fractions of neurons compared to the NAP specifications computed by previous work, yet they can significantly expand the verifiable boundaries to several orders of magnitude larger.

\end{abstract}

\received{20 February 2007}
\received[revised]{12 March 2009}
\received[accepted]{5 June 2009}

\maketitle

\input{intro}
\input{background}
\input{motivation}
\input{method}

\input{evaluation}

\input{related}

\input{conclusion}
\pagebreak

\bibliographystyle{ACM-Reference-Format}
\bibliography{main}

% \bibliography{references}
% \bibliographystyle{icml2023}

%%%%%%%%%%%%%%%%%%%%%%%%%%%%%%%%%%%%%%%%%%%%%%%%%%%%%%%%%%%%%%%%%%%%%%%%%%%%%%%
%%%%%%%%%%%%%%%%%%%%%%%%%%%%%%%%%%%%%%%%%%%%%%%%%%%%%%%%%%%%%%%%%%%%%%%%%%%%%%%
% APPENDIX
%%%%%%%%%%%%%%%%%%%%%%%%%%%%%%%%%%%%%%%%%%%%%%%%%%%%%%%%%%%%%%%%%%%%%%%%%%%%%%%
%%%%%%%%%%%%%%%%%%%%%%%%%%%%%%%%%%%%%%%%%%%%%%%%%%%%%%%%%%%%%%%%%%%%%%%%%%%%%%%
\newpage
\appendix
\onecolumn

% You can have as much text here as you want. The main body must be at most $8$ pages long.
% For the final version, one more page can be added.
% If you want, you can use an appendix like this one, even using the one-column format.

\input{appendix.tex}

% \appendix









\end{document}

%% file: macro-defs.tex
\newcommand{\pertx}{B(x, \epsilon)}

%% file: intro.tex
\section{Introduction}

The growing prevalence of deep learning systems in decision-critical applications has elevated safety concerns regarding AI systems, such as their vulnerability to adversarial attacks ~\cite{GoodfellowSS14, DBLP:journals/cacm/DietterichH15}. 
Therefore, the verification of AI systems has become increasingly important and attracted much attention from the research community. The field of neural network verification largely follows the paradigm of software verification -- using formal methods to verify desirable properties of systems through rigorous mathematical specifications and proofs~\cite{DBLP:journals/computer/Wing90}. A notable trend in existing works~\cite{reluplex, Marabou, huang2017cav, huang2020csr, abc} focuses on scaling verification to larger and more complex neural networks. While scalability is undeniably important and requires the collective effort of the research community, this paper explores an orthogonal angle that has been largely overlooked: \textit{defining meaningful specifications}.

\begin{figure}[ht]
    \centering
    % First row
    \begin{subfigure}[t]{0.48\textwidth}
    {\includegraphics[width=\linewidth]{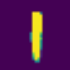}}
        \caption{Local neighborhood as specifications of the 3 reference points compared to using NAP as specification, NAP learnt from the 3 reference points and their local neighborhoods. Covers 4.23\% and 7.64\% of the WL region respectively.}
        \label{fig:subfig1}
    \end{subfigure}
    \hfill
    \begin{subfigure}[t]{0.48\textwidth}
        {\includegraphics[width=\linewidth]{figures/3_points_nap_combined.pdf}}
        \caption{NAP as specification, NAP learnt from the 3 reference points. Covers 10.61\% of the WL region.}
        \label{fig:subfig2}
    \end{subfigure}
    
    \vspace{1em} % Add space between rows

    % Second row
    \begin{subfigure}[t]{0.48\textwidth}
        {\includegraphics[width=\linewidth]{figures/dataset_nap.pdf}}
        \caption{NAP as specification, NAP learnt from data points in the entire WL region. Covers 26.36\% of the WL region.}
        \label{fig:subfig3}
    \end{subfigure}
    \hfill
    \begin{subfigure}[t]{0.48\textwidth}
        {\includegraphics[width=\linewidth]{figures/coarsened.pdf}}
        \caption{NAP as specification, with NAP learnt from the entire WL region and coarsened. Covers 32.43\% of the WL region.}
        \label{fig:subfig4}
    \end{subfigure}

    % Main caption
    \caption{Illustration of verifiable regions of the Weak Left (WL) region of the ACAS\_Xu \cite{acas_xu} advisories for a head-on encounter with \( a_{\text{prev}} = \text{Clear of Conflict (COC)}, \, \tau = 0\text{s}\). The red bounding boxes represent the verifiable regions when using data as specifications, constrained by the reference points and their adversarial examples. The green hatched areas denote verifiable regions when using Neural Activation Patterns (NAP) as specifications. }
    \label{fig:ACAS_XU}
\end{figure}

To illustrate, nearly all existing works follow a \textit{"data as specification"} paradigm, where the specification is defined by the consistency of local neighborhoods—often \(L_\infty\) balls—around reference data points. While local neighborhood specification provides a fair and effective testbed for evaluating neural network robustness, it primarily covers a tiny convex region of input data that can mathematically be described by adding noise to the reference point, as illustrated by the red bounding box in Figure \ref{fig:subfig1}. This small region is constrained by it's adversarial examples. It is too restrictive to adequately address a broader, unseen test set, which are real data sampled from the underlying distribution. 
% \allen{For instance, the maximum \(L_\infty\) verifiable bounds used in VNNCOMP \cite{vnncomp2023} – the annual neural network verification competition – are usually less than 0.2, while the smallest distance between data points with the same label exceeds 0.5.
% Indeed, due to the nature of image distribution, the distances between real images are far beyond \(L_\infty\) verifiable bounds. As a result, local neighborhood specifications are not suitable for the verification of unseen test data—a challenging task with significant real-world implications.}

% Otherwise, one could easily solve machine learning problems by writing down such mathematical formulas.

% To illustrate this, consider one of the most commonly used verification benchmarks, the MNIST dataset: the maximum \(L_\infty\) verifiable bounds in the MNIST dataset are less than 0.2, while the smallest distance between images with the same label exceeds 0.5.

Ideally, the specification should produce a verifiable region that includes all data points from the same class, where the data points may be distributed in a non-linear and non-convex manner within the input space.
For instance, consider the safety-critical Airborne Collision Avoidance System for Unmanned Aircraft (ACAS Xu), the neural network processes an input representing the state of the aircraft and outputs one of five advisories: Clear of Conflict (COC), Weak Left (WL), Strong Left (SL), Weak Right (WR), or Strong Right (SR). Observe in \ref{fig:ACAS_XU} the the inputs belonging to Weak Left (WL) depicted in grey forms a irregular shape. This complexity underscores the need for flexible and robust specifications capable of accurately capturing such intricate distributions. Unlike local neighborhood specifications, manually defining such a specification is impractical. This poses a tricky chicken-egg problem: machine learning (ML) is necessary because it's challenging to formally write down a precise definition (aka specification); but to be able to verify machine learning models, a formal specification would be needed. We argue that, in order to tackle this challenge, a separate learning algorithm for specifications is necessary. In this view, the \textit{ “data as specification”} paradigm is a simple but extremely overfitted algorithm for specification learning, which simply picks a small neighborhood of a reference data point in the input space. To this end, \citet{geng23} proposes using a new and more promising paradigm - \textit{"neural representation as specification"}, which learns a specification in the representation space of the trained machine learning model in the form of neural activation patterns (NAPs). NAPs are value abstractions of hidden neurons which have been shown useful for understanding the inference of a model~\cite{nap_exe}. Most importantly, a well-trained neural network would exhibit similar activation patterns for input data from the same class, regardless of their actual distance in the input space~\cite{rep_learning,tishby2015deep,geng23}. This key observation suggests that if we learn a NAP—a common activation pattern shared by a certain class of data—it can serve as a candidate specification for verifying data from that class. Once such a NAP specification is successfully verified, we say any data covered by this NAP (exhibiting this pattern) provably belongs to the corresponding class. Ideally, if that NAP covers all data from a class, it can be considered a machine-checkable definition of that class. Geometrically speaking, compared to the \(L_\infty\) ball specifications, NAP specifications provide greater flexibility, enabling coverage of larger and more irregular regions, This advantage is particularly notable when reference points are near the boundaries of the WL region as illustrated in Figure \ref{fig:subfig1}. In \ref{fig:subfig2} we observe that the common NAP learned from the three reference points covers a significantly larger region compared to the combined coverage of the individual NAPs learned from each reference point independently. \ref{fig:subfig3} demonstrates even greater coverage when the NAP specification is learned from all inputs belonging to the WL class.

However, it is noteworthy that the current approach to computing NAPs relies on a straightforward statistical method that assumes each neuron contributes to certifying the robustness of neural networks. Consequently, the computed NAPs are often overly refined. This is a restrictive assumption, as many studies \cite{lottery,DyingReLU,Pruning,RedundancyReduction} have shown that a significant portion of neurons in neural networks may not play a substantial role. In the spirit of Occam's Razor, we aim to systematically remove these redundant neurons that do not impact robustness. This motivates us to address the following challenge: given a neural network, how can we identify a minimal NAP (i.e., the coarsest abstraction) that is sufficient for verifying the network's robustness? This problem is important for the following reasons: i) Minimal NAP specifications cover potentially much larger regions in the input space compared to the most refined ones, enhancing the ability to verify more unseen data; ii) Minimal NAP specifications provide insight into which neurons are responsible for the model’s robust predictions, helping to uncover the black-box nature of neural networks. For instance, if we aim to decode NAPs into human-understandable programs or rules \cite{rule_li}, minimal NAPs will always be easier to interpret than larger NAPs. We leave the interpretation of individual neurons as future work.

To find the minimal NAP specifications, we first introduce a basic algorithm, \textsc{Coarsen}, which exhaustively checks all possible candidates using off-the-shelf verification tools, such as Marabou \cite{Marabou}. Specifically, \textsc{Coarsen} gradually coarsens each neuron in the most refined NAPs, retaining only the coarsened neurons when Marabou confirms verification success. While this approach provides correctness guarantees, it is not efficient for verifying large neural networks, as calls to verification tools are typically expensive. To improve efficiency, we further propose statistical variants of \textsc{Coarsen}—namely, \textsc{StochCoarsen}, which leverages sampling and statistical learning principles to find minimal NAP specifications. \ref{fig:subfig4} highlights that coarsened NAP specifications, which involve fewer neurons, can cover an even larger portion of the WL region, highlighting the effectiveness of coarsened specifications in achieving broader coverage of the complex input space.

However, verification-dependent approaches face challenges in scaling up to state-of-the-art neural network models due to the limitations of current verification tools. To address this, we develop an optimistic method, \textsc{OptAdvPrune}, which leverages adversarial examples to identify essential neurons—the fundamental building blocks of minimal NAPs. Our experimental results show it can efficiently estimate an initial starting point for minimal NAP specification and only making calls to the verification tool until the very end. We further apply these methods to state-of-the-art neural networks, such as VGG-19 \cite{vgg}. While formal verification of these estimated minimal NAP specifications remains impractical due to current scalability constraints in the underlying verification engine, we demonstrate that these NAPs capture significant hidden features and concepts learned by the model. As many studies suggest, visual interpretability and robustness are inherently related, emerging in learned features and representations \cite{alvarez2018towards,boopathy2020proper}. Consequently, we believe that the identified essential neurons contribute to the model's robustness, and their activation states can serve as empirical indicators of confident predictions. Our contributions can be summarized as follows:

\begin{itemize}
\item We introduce the problem of learning minimal NAP specifications, emphasizing the need for a new paradigm in neural network specification. We present three approaches, each offering distinct trade-offs between efficiency and performance.

\item We define key concepts such as the abstraction function, NAP specification, and essential neurons. We demonstrate that the problem reduces to identifying essential neurons and propose an optimistic approaches for estimating them.

\item We propose a simple yet effective method for estimating the volume of the region mapped by the NAP. Our experiments show that minimal NAP specifications extend the verifiable bound by several orders of magnitude.

\item We estimate essential neurons in the state-of-the-art VGG-19 network. Using a modified Grad-CAM map, we show that these essential neurons contribute to visual interpretability, providing strong evidence that they may also explain the model's robustness.
\end{itemize}

\input{FigFailureOfDistance}

%% file: background.tex
\section{Background}
\label{sec:background}

In this section, we introduce basic knowledge and notations of adversarial attacks and neural network verification, with an emphasis on verification using NAP specifications. This may help readers better understand the importance of learning minimal NAP specifications.

\subsection{Neural Networks for Classification Tasks}
\label{sec:NN_basic}

% In this paper, we focus on feed-forward ReLU neural networks. Generally speaking, a feed-forward network $N$ is comprised of $L$ layers, where each layer performs a linear transformation followed by a ReLU activation. We denote the pre-activation value and post-activation value at the $l$-th layer as $z^{(l)}(x)$ and $\hat{z}^{(l)}(x)$, respectively. The $l$-th layer computation is expressed as follows: $z^{(l)}(x) = \mathbf{W}^{(l)} \hat{z}^{(l-1)}(x) + \mathbf{b}^{(l)}$, $\hat{z}^{(l)}(x) = \mathbf{ReLU}(z^{(l)}(x))$, with $\mathbf{W}^{(l)}$ being the weight matrix and $\mathbf{b}^{(l)}$ representing the bias for the $l$-th layer. We denote the number of neurons in the $l$-th layer as $d_l$, and the $i$-th neuron in layer $l$ as $N_{i, l}$. The pre-activation value and post-activation value of $N_{i, l}$ at input $x$ is computed by $z_i^{(l)}(x)$ and $\hat{z}_i^{(l)}(x)$. The network can also be viewed as a function $F^{<N>} := \mathbb{R}^{d_0} \rightarrow \mathbb{R}^{d_L}$ such that $F^{<N>}(x) := z^{(L)}(x)$, where $F^{<N>}_i(x) := z_i^{(L)}(x)$ represents the output of $i$-th neuron in the last layer. We will omit $N$ when the context is clear. In multi-class classification with classes \( C \), given input \( x \), the neural network predicts class \( c \in C \) if the output \( \textbf{F}_c(x) \) of the \( c \)-th neuron in the \( L \)-th layer is the highest.

In this paper, we focus on feed-forward neural networks equipped with ReLU activation functions. A feed-forward network \( N \) consists of \( L \) layers, where each layer alternates between a linear transformation and a ReLU activation. For the \( l \)-th layer, we denote the pre-activation and post-activation values as \( z^{(l)}(x) \) and \( \hat{z}^{(l)}(x) \), respectively. These are computed as follows: $
z^{(l)}(x) = W^{(l)} \hat{z}^{(l-1)}(x) + b^{(l)}, \quad \hat{z}^{(l)}(x) = ReLU(z^{(l)}(x))$, where \( W^{(l)} \) and \( b^{(l)} \) represent the weight matrix and bias vector of the \( l \)-th layer, respectively.

Let \( d_l \) denote the number of neurons in the \( l \)-th layer, and let \( N_{i,l} \) refer to the \( i \)-th neuron in that layer. The pre-activation and post-activation values of \( N_{i,l} \) for a given input \( x \) are \( z_i^{(l)}(x) \) and \( \hat{z}_i^{(l)}(x) \), respectively. The network as a whole can be viewed as a function \( F^{<N>}: \mathbb{R}^{d_0} \to \mathbb{R}^{d_L} \), where \( F^{<N>}(x) = z^{(L)}(x) \) represents the output of the final layer. For the \( i \)-th neuron in the last layer, its output is \( F^{<N>}_i(x) = z_i^{(L)}(x) \). When the context is clear, we omit \( N \) for simplicity and refer to the network as \( F \). In multi-class classification tasks with \( C \) classes, the network predicts the class \( c \in C \) for a given input \( x \) if the output \( F_c(x) \) of the \( c \)-th neuron in the \( L \)-th layer is the largest among all outputs.

\subsection{Adversarial Attacks and Neural Networks Verification Problem}
% \ag{Rewrite this paragraph to separate definition from discussion. First define, then discuss if necessary}

Given a neural network $F$ and a reference point $x$, adversarial attacks aim to search for a point $x'$ that is geometrically close to the reference point $x$ such that $x'$ and $x$ belong to different classes. Here, we use the canonical specification, that is, we want to search in the local neighborhoods ($L_\infty$  norm balls) of $x$, formally denoted as $\pertx := \{x' \mid ||x-x'||_{\infty} \leq \epsilon \}$, where $\epsilon$ is the radius. For certain $\epsilon$, given we know $x$ belongs to class $j$, we say an adversarial point is found if:
\begin{align}
\label{eq:robustness_ori}
    \exists x' \in \pertx \quad  \exists i \in C \quad F_i(x') - F_j(x') > 0
\end{align}
In practice, the change from the original data $x$ to adversarial data $x'$ should be imperceptible, so they are more likely to be recognized as the same class/label from a human perspective. There are also metrics other than the $L_\infty$ norm to represent the "similarity" between $x$ and $x'$, such as the $L_0$ and $L_2$ norms \citep{Xu2020AdversarialAA}. However, almost all of them fall into the local neighborhood specification paradigm. This is different from the NAP specification, as we will discuss later.

Neural networks are vulnerable to adversarial attacks, where even imperceptible changes can alter predictions significantly. This underscores the critical need for neural network verification, which can be viewed as highlighting safe regions that excludes adversarial regions. Solving the verification problem involves formally proving the absence of adversarial points in $\pertx$. Formally, the verification problem seeks to prove:
\begin{align}
\label{eq:robustness_ori}
    \forall x' \in \pertx \quad \forall i \neq j \quad F_j(x') - F_i(x') > 0
\end{align}
For a simpler presentation, we assume that $F(x)$ is a binary classifier. For any given specification $\pertx$, $F(x) \geq 0$ indicates that the model is verified; otherwise, we can find an adversarial example. Solving such a verification problem is known to be NP-hard \cite{reluplex}, and achieving scalability in verification remains an ongoing challenge.

\subsection{Neuron Abstractions and Neural Activation Patterns}
To better discuss NAP specification and robustness verification, we first introduce the relevant concepts of neuron abstractions, neuron abstraction functions, and neural activation patterns.
% \vspace{-5pt}

\paragraph{Neuron Abstractions} For a neural network \( N \), an internal neuron \( N_{i,l} \) (where \( 0 \leq i \leq d_l \) and \( 1 \leq l \leq L-1 \)) can have its post-activation value \( \hat{z}_i^{(l)}(x) \) abstracted into finite states. This abstraction is a mapping from \( \mathbb{R} \) to a set \( \mathbb{S} = \{s_1, s_2, \ldots, s_n\} \), representing neuron states. A simple binary abstraction for ReLU activation defines two states: \( s_0 := 0 \) (deactivation) and \( s_1 := (0, \infty) \) (activation). Further, these states can be abstracted into a unary state \( s_* = [0, \infty) \), covering the entire range of post-activation. This leads to a partial order: \( s_* \preceq s_0 \) and \( s_* \preceq s_1  \), where \( s_0 \) and \( s_1 \) refine \( s_* \). For simplicity, we use \( \mathbf{0}, \mathbf{1}, \mathbf{*} \) to represent these states. Our specification learning approach applies to any activation function definable within the abstraction domain, though this study focuses on ReLU.

\vspace{-5pt}
% r instance, if a neuron appears to be in state $s_0$ most of the time, we fix this neuron being in state $s_0$. Similar logic also applies to state $s_1$. In addition, the coarsest abstraction of $\hat{Z}{j}(x)$ is $s_*$, where $s_*:= [0, \infty)$.
% Suppose we feed a collection of inputs to a neural network and observe a specific neuron's behavior, 

% $\mathbf{0}$
% $\mathbf{1}$
% $\mathbf{*}$

\begin{definition}[Neuron Abstraction Function] Given a neural network $N$ and the abstraction state set $\mathbb{S}$. A neuron abstraction function is a mapping $\mathcal{A}: N \rightarrow \mathbb{S}$. Formally, for an arbitrary neuron $N_{i, l}$, the function abstracts $N_{i, l}$ to some state $s_k \in \mathbb{S}$, i.e., $\mathcal{A}(N_{i, l}) = s_k$.
\end{definition}

% \{\mathbf{0}, \mathbf{1}, \mathbf{*}\
% where $ 0 \leq i \leq d_l$ and $1 \leq l \leq L-1$

The above characterization of neuron abstraction does not instruct us on how to perform binary abstraction in the absence of neuron values. Therefore, we introduce two types of $\mathcal{A}$: unary $\dot{\mathcal{A}}$ and binary $\ddot{\mathcal{A}}$, both of which take a single input $x$ as a parameter, omitting it when the context is clear.

\paragraph{The unary $\dot{\mathcal{A}}$ function} $\dot{\mathcal{A}}$ always maps any input to the coarsest state $\mathbf{*}$. Formally:
\begin{align}
\label{eq:id_func}
\dot{\mathcal{A}}(N_{i, l}, x ) = \mathbf{*}
\end{align}
\paragraph{The binary $\ddot{\mathcal{A}}$ function} When an input $x$ is passed through an internal neuron $N_{i, l}$, we can easily determine the binary abstraction state of $N_{i, l}$ based on its post-activation value $\hat{z}_i^{(l)}(x)$. This motivates us to define the binary abstraction function as follows: 
\begin{align}
\label{eq:id_func}
\ddot{\mathcal{A}}(N_{i, l}, x ) =
\begin{cases}
  \mathbf{0} & \text{if }  \hat{z}_i^{(l)}(x) = 0 \\
  \mathbf{1} & \text{if } \hat{z}_i^{(l)}(x) > 0 \\ 
\end{cases}
\end{align}

\begin{figure}[t]
    \centering
     \begin{subfigure}[t]{0.15\textwidth}
         \centering         
         \includegraphics[width=\textwidth]{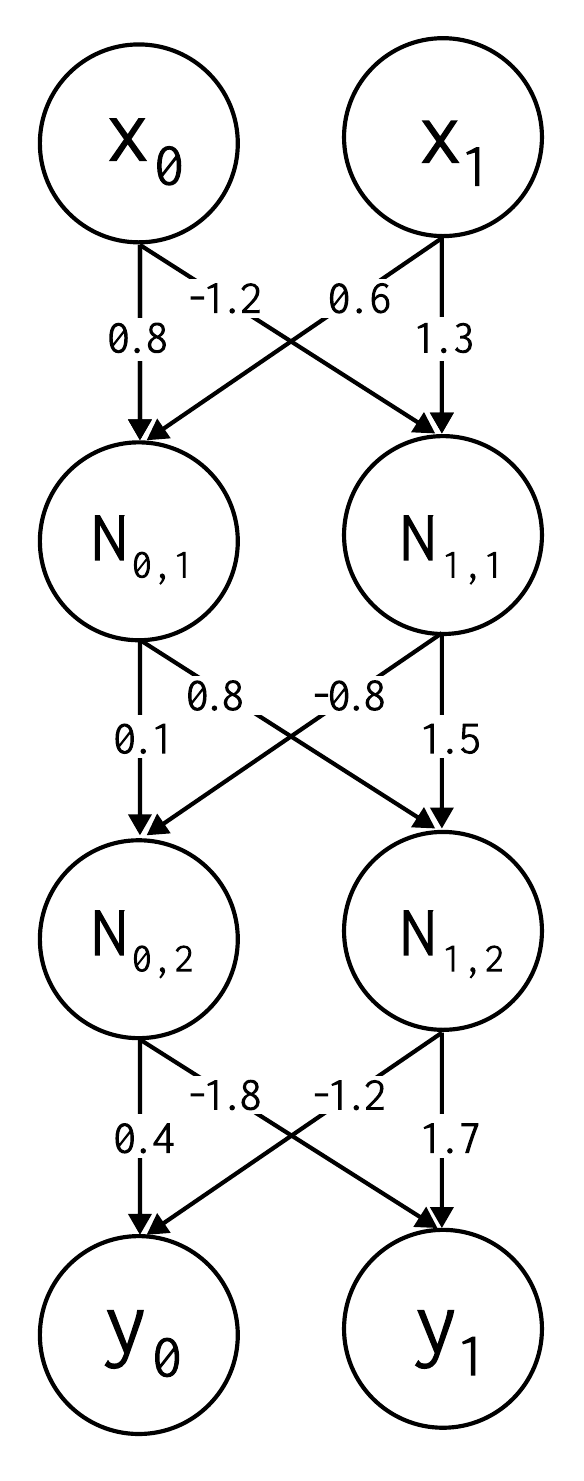}
         \caption{A simple 2x2 fully connected neural network.}
         \label{fig:model_arch}
     \end{subfigure}
     \hfill
     \begin{subfigure}[t]{0.84\textwidth}
         \centering         \includegraphics[width=\textwidth]{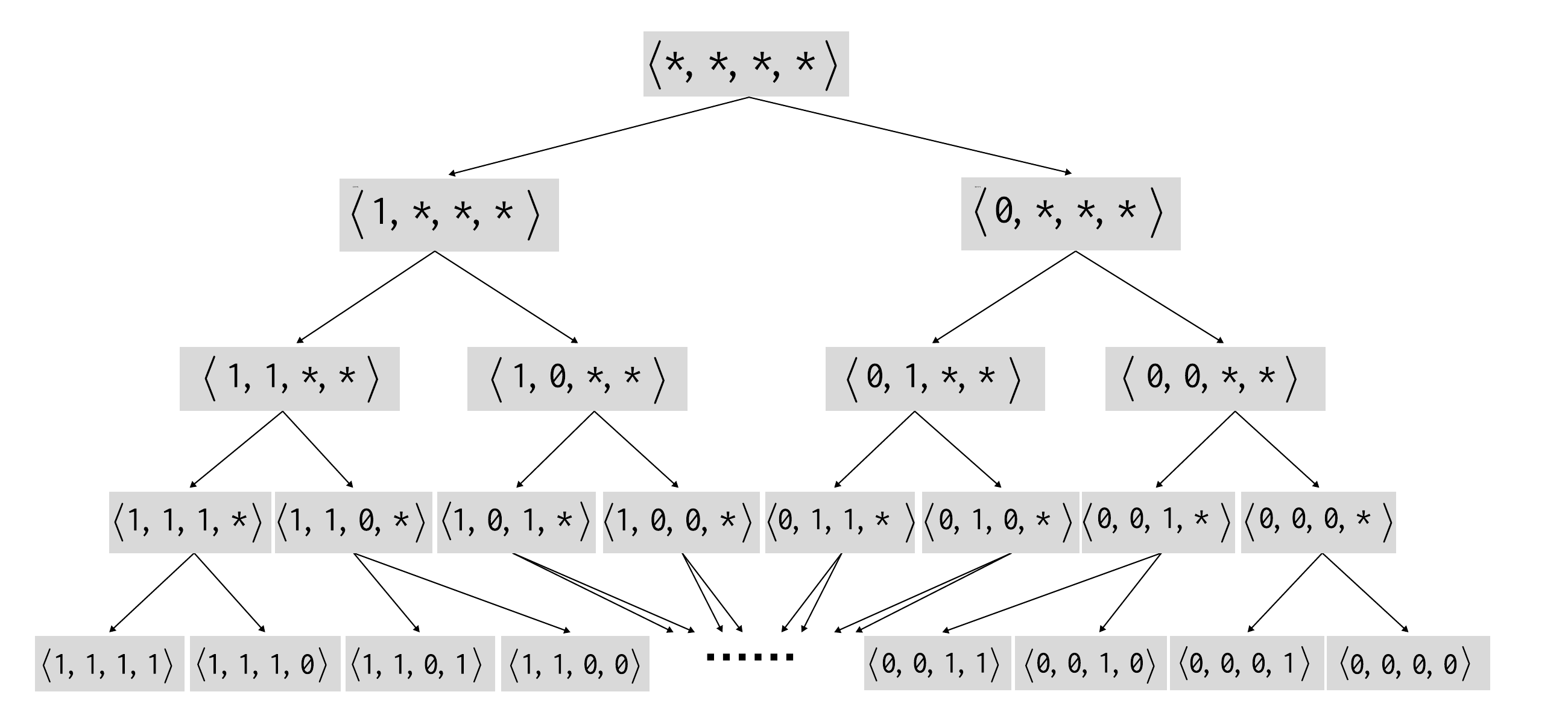}
         \caption{A subset of the NAP $\mathcal{P}$ family forms a binary tree. Each child subsumes its parent, and the root is the most abstracted NAP. The leaf node represents the most refined NAPs, where each neuron is either $\mathbf{1}$ or $\mathbf{0}$. From the root, the tree is spanned by abstracting each neuron in an order of $N_{0,1}, N_{1,1}, N_{0,2}, N_{1,2}$. Setting a different order will create a different tree of NAPs. } 
         \label{fig:binary_tree}
         \hfill
     \end{subfigure}
     \caption{A simple 2x2 fully connected neural network and a subset of its NAP in a binary tree structure.}
\end{figure}

% \begin{figure}[!h]
%          \centering         \includegraphics[width=1\textwidth]{figures/binary_tree.pdf}
%          \caption{A subset of the NAP $\mathcal{P}$ family forms a binary tree. Each child subsumes its parent, and the root is the most abstracted NAP. The leaf node represents the most refined NAPs, where each neuron is either $\mathbf{1}$ or $\mathbf{0}$. From the root, the tree is spanned by abstracting each neuron in an order of $N_{0,1}, N_{1,1}, N_{0,2}, N_{1,2}$. Setting a different order will create a different tree of NAPs. } 
%          \label{fig:binary_tree}
% \end{figure}

 % \caption{Visualization of a simple 2x2 fully connected neural network architecture used for classification tasks and regions in input space bounded by different NAPs.}

% For real datasets, we can determine a neuron's abstraction state using either $\ddot{\mathcal{A}}$ or $\dot{\mathcal{A}}$ based on the approaches described above. Formally, the state of $N_{i, l}$ given class $c$ is denoted as $\mathcal{A}(N_{i, l}, X_c)$, where $X_c$ represents inputs associated with the class $c$ and $\mathcal{A} \in \{ \ddot{\mathcal{A}}, \dot{\mathcal{A}} \}$.

% \rebecca{which A to use, one dot or two dots? Does this notation imply that both As are used? Same for Line 293. } 

\vspace{-5pt}
% \begin{definition}[Neural Activation Pattern] Given a neural network $N$, an input $x$, and neuron abstraction functions $\{ \ddot{\mathcal{A}}, \dot{\mathcal{A}} \}$.  
% A neural activation pattern (NAP) $P$ is a tuple that consists of the abstraction state of all neurons in $N$. Formally, $P := \langle \mathcal{A}(N_{i, l},x) \text{ } | \text{ } N_{i, l} \in N, \mathcal{A} = \ddot{\mathcal{A}} \text{ } or \text{ } \dot{\mathcal{A}} \}\rangle$, also denoted as $\mathcal{A}(N)$. The neuron $N_{i, l}$'s abstraction state specified by NAP $P$ is represented as $P_{i, l}$, i.e., $P_{i, l} = \mathcal{A}(N_{i, l})$. This mainly repsresnet a NAP triggered by input $x$. 
% \end{definition}

\begin{definition}[Neural Activation Pattern] 
Given a neural network $N$, an input $x$, and neuron abstraction functions $\ddot{\mathcal{A}}$ and $\dot{\mathcal{A}}$, a neural activation pattern (NAP) $P$ is a tuple representing the abstraction state of all neurons in $N$. Formally, $P := \langle \mathcal{A}(N_{i, l},x) \mid N_{i, l} \in N \rangle$, where $\mathcal{A}$ is either $\ddot{\mathcal{A}}$ or $\dot{\mathcal{A}}$, chosen for each neuron. We also denote \( P \) as \( \mathcal{A}(N, x) \). The abstraction state of neuron \( N_{i, l} \) is represented as \( P_{i, l} \), specifically \( P_{i, l} = \mathcal{A}(N_{i, l}, x) \). When representing a refined activation pattern triggered by \( x \) where \( \ddot{\mathcal{A}} \) is applied to each neuron, it is expressed as \( \ddot{\mathcal{A}}(N, x) \). 
\end{definition}

% \begin{definition}[Neural Activation Pattern] 
% Given a neural network $N$, an input $x$, and neuron abstraction functions $\ddot{\mathcal{A}}$ and $\dot{\mathcal{A}}$, a neural activation pattern (NAP) $P$ is a tuple representing the abstraction state of all neurons in $N$. Formally, 
% \[
% P := \langle \mathcal{A}(N_{i, l},x) \mid N_{i, l} \in N \rangle,
% \]
% where $\mathcal{A}$ is either $\ddot{\mathcal{A}}$ or $\dot{\mathcal{A}}$, chosen for each neuron. The abstraction state of neuron $N_{i, l}$ is denoted as $P_{i, l}$, i.e., $P_{i, l} = \mathcal{A}(N_{i, l}, x)$. When representing a refined activation pattern triggered by $x$, where $\ddot{\mathcal{A}}$ is applied to each neuron, it is expressed as $\ddot{\mathcal{A}}(N_{i, l}, x)$.
% \end{definition}

% $P := \langle \mathcal{A}(N_{i, l},x) \text{ } | \text{ } N_{i, l} \in N, \mathcal{A} \in \{ \ddot{\mathcal{A}}, \dot{\mathcal{A}} \}\rangle$

% \rebecca{Definition 2.2 presents a conceptual tool, where the unary and binary abstraction functions take a single input x. This definition demonstrates that NAPs are merely abstraction states of neurons. Definition 2.4 introduces the class of NAPs and explains how a family of NAPs is created. In short, it uses the statistical abstraction function \Tilde{A} and a set of inputs X to compute a family of NAPs.}

We denote the power set of NAPs in $N$ as $\mathcal{P}$. The number of all possible NAPs in $N$ scales exponentially as the number of neurons increases. For such a large set, if we aim to find the minimal NAP --- the central problem in this work, we first have to establish an order so that NAPs can be compared. To this end, we define the following partial order: 

% A partial order on a set $\mathcal{P}$ is a relation $\preceq$ on $S$ that satisfies the following properties:

% \begin{enumerate}
%     \item \textbf{Reflexivity:} For every element $x$ in $S$, $x \preceq x$ (the reflexive property).
%     \item \textbf{Antisymmetry:} If $x \preceq y$ and $y \preceq x$, then $x = y$ (the antisymmetric property).
%     \item \textbf{Transitivity:} If $x \preceq y$ and $y \preceq z$, then $x \preceq z$ (the transitive property).
% \end{enumerate}

\begin{definition}[Partially ordered NAP] For any given two NAPs $P, P' \in \mathcal{P}$, we say $P'$ subsumes $P$ if, for each neuron $N_{i, l}$, its state in $P$ is an abstraction of that in $P'$. Formally, this can be defined as:
\begin{align}
    P' \preccurlyeq P \text{ } \iff \text{ }  P'_{i, l} \preceq P_{i, l} \text{ }  \text{ }  \forall N_{i, l} \in N
\end{align}

Moreover, two NAPs $P, P'$ are equivalent if 
$P \preccurlyeq P'$ and $P' \preccurlyeq P$.
\end{definition}

% We train XNET on $1\,000$ randomly-generated inputs and achieve the perfect F1-score of 1 on another $1\,000$ random inputs.
% \ag{this ends abruptly. better merge it where it is actually used as an example. This is not a running example at this point since it is not detailed enough.}

To give a concrete example, Figure \ref{fig:binary_tree} depicts a subset of all possible NAPs of a simple neural network consisting of 2 hidden layers and 4 neurons, as presented in Figure \ref{fig:model_arch}. It is noteworthy that a subset of the NAP family \( \mathcal{P} \) can form a complete binary tree based on the order of abstraction in neurons. For example, using the order \( N_{0,1}, N_{1,1}, N_{0,2}, N_{1,2} \) creates a specific tree of NAPs; different orders will yield different trees. The root of this tree is the coarsest NAP, \( \langle \mathbf{*}, \mathbf{*}, \mathbf{*}, \mathbf{*} \rangle \). Increasing the tree's depth means that \( \ddot{\mathcal{A}} \) applies to more neurons, and at the leaf nodes, all neurons are abstracted by \( \ddot{\mathcal{A}} \). The leaf nodes represent the most refined NAPs, totaling \( 2^{|N|} \). Additionally, each parent node subsumes its child, meaning a leaf node will always be subsumed by its ancestors along the path. For example, we have the relationship: 
\[
\langle \mathbf{*}, \mathbf{*}, \mathbf{*}, \mathbf{*} \rangle \preccurlyeq 
\langle \mathbf{1}, \mathbf{0}, \mathbf{*}, \mathbf{*} \rangle \preccurlyeq 
\langle \mathbf{1}, \mathbf{0}, \mathbf{1}, \mathbf{*} \rangle \preccurlyeq 
\langle \mathbf{1}, \mathbf{0}, \mathbf{1}, \mathbf{0} \rangle.
\]
However, children under the same parent are not comparable, as seen in \( \langle \mathbf{1}, \mathbf{0}, \mathbf{*}, \mathbf{*} \rangle \undernegpreccurlyeq \langle \mathbf{1}, \mathbf{1}, \mathbf{*}, \mathbf{*} \rangle \). This occurs because these NAPs reside on different branches formed by splitting certain ReLUs.
\begin{figure}[h]
    \centering
  \includegraphics[width=0.8\textwidth]{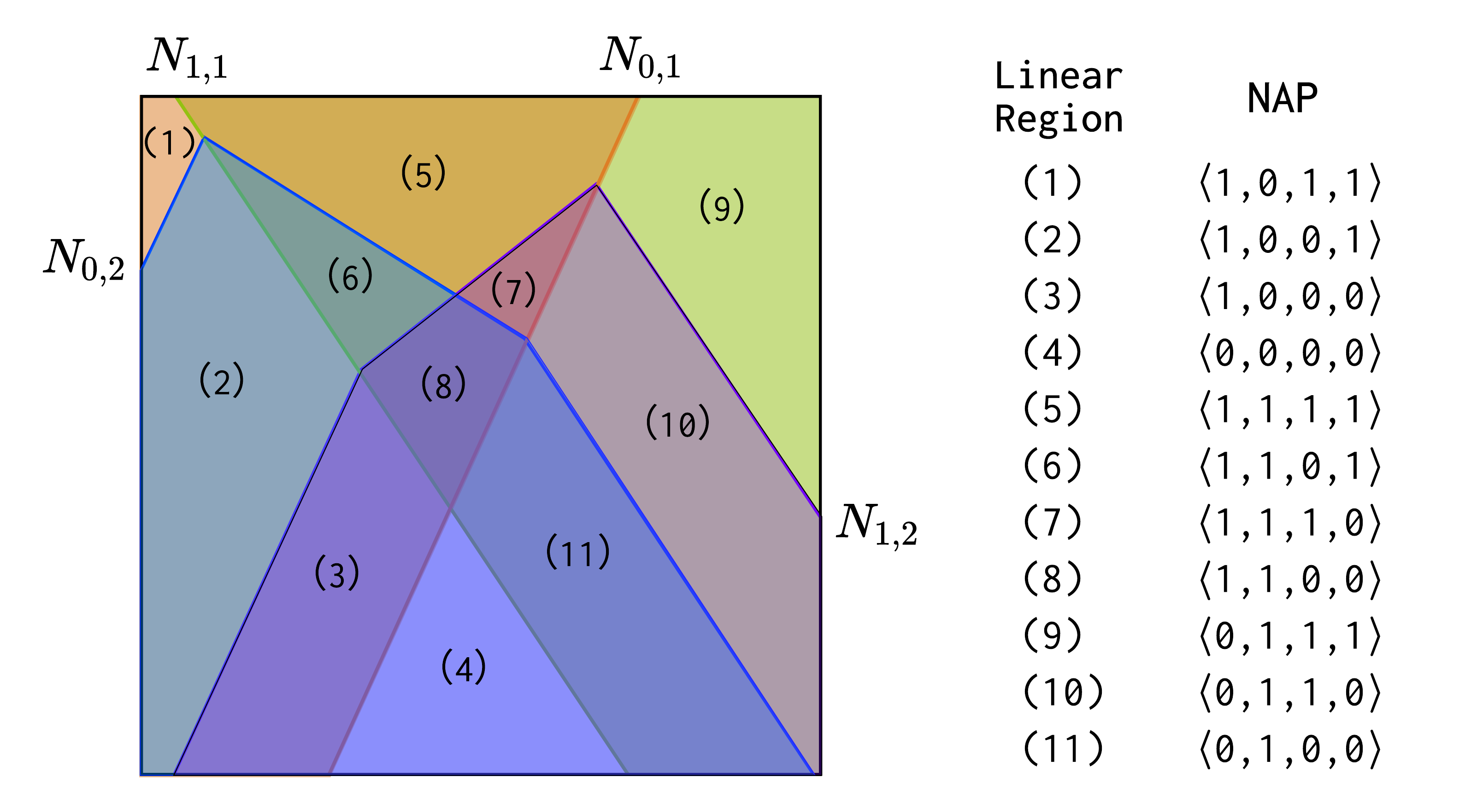}    \caption{Linear regions are shattered by the simple 2X2 neural network. Each linear region corresponds to a most refined NAP, but not necessarily vice versa.  More abstracted NAPs are formed by ignoring lines/hyperplanes created by neurons.}
    \label{fig:NAP_regions}
\end{figure}
\vspace{-1em}
\paragraph{Regions Outlined by NAPs} A key requirement for verification specifications is their ability to represent specific regions within the input space. Canonical local neighbor specifications define \( L_\infty \) norm balls using explicit formulas, such as \( \pertx := \{x' \mid ||x-x'||_{\infty} \leq \epsilon \} \). In contrast, NAP specifications implicitly outline certain regions in the input space. We define the regions specified by \( P \) as \( R_P \), which represents the set of inputs whose activation patterns subsumed by the given NAP \( P \). Formally, \( R_P := \{ x \mid P \preccurlyeq \ddot{\mathcal{A}}(N,x) \} \). For instance, Figure \ref{fig:NAP_regions} illustrates the NAP family of the simple neural network shown in Figure \ref{fig:model_arch}. These NAPs correspond to regions bounded by hyperplanes created by neurons, with the most refined NAPs representing individual linear regions from (1) to (11). For example, linear region (9) corresponds to \( \langle\mathbf{0},\mathbf{1},\mathbf{1},\mathbf{1}\rangle \). The coarsest state \( \mathbf{*} \) abstracts the binary states \( \mathbf{0} \) and \( \mathbf{1} \), allowing a NAP with more \( \mathbf{*} \) to cover a larger, potentially concave region in the input space. For instance, the NAP \( \langle\mathbf{*},\mathbf{*},\mathbf{1},\mathbf{*}\rangle \) corresponds to the union of linear regions (1), (5), (7), (9), and (10). Notably, the number of linear regions is less than the size of the NAP family, as reported in \cite{Geng_scalar,LR1,LR2}.

\subsection{NAP Specifications for Robustness Verification}
\label{sec:NAP_property}
The previously introduced NAP definition primarily serves as a conceptual tool. To function as a specification for robustness verification, the regions outlined by NAPs (\( R_P \)) should cover a significant amount of data from a specific class. We demonstrate how to achieve this using the $\widetilde{\mathcal{A}}$ function.
% subsection illustrates how NAP specifications can be utilized for robustness verification. We discuss how the abstraction states of neurons can serve as signatures of distinct classes with the introduction of the class NAP definition. 

\paragraph{The statistical $\widetilde{\mathcal{A}}$ function}
The binary $\ddot{\mathcal{A}}$ function is limited to processing a single input and encounters challenges when multiple inputs are involved, as two distinct inputs may lead to conflicting abstraction states for a neuron; for instance, a neuron may be deactivated for input \( x_1 \) but activated for another input \( x_2 \). This can be problematic when attempting to learn a general NAP for an entire class of inputs. To address this issue, we approach it statistically by introducing the $\widetilde{\mathcal{A}}$ function, defined as follows:

% The binary $\ddot{\mathcal{A}}$ function is limited in calculating a single input, and meets a challenge when multiple inputs come into play, as two distinct inputs may lead to disagreements in a neuron's abstraction state \rebecca{explain disagreements}. We then approach this problem statistically by introducing the $\widetilde{\mathcal{A}}$ function defined as follows \rebecca{how does 0,1, * connect to partial order in line 223-224?}:
\begin{align}
\label{eq:stat_func}
 \widetilde{\mathcal{A}}(N_{i, l}, X) =
\begin{cases}
  \mathbf{0} & \text{if } \frac{|\{x_j|\ddot{\mathcal{A}}(N_{i, l}, x_j ) = \mathbf{0}, \text{ } x_j \in X\}|}{|X|} \geq \delta \\
  \mathbf{1} & \text{if } \frac{|\{x_j|\ddot{\mathcal{A}}(N_{i, l}, x_j ) = \mathbf{1}, \text{ } x_j \in X\}|}{|X|} \geq \delta \\ 
  \mathbf{*} & \text{otherwise}\\ 
\end{cases}
\end{align}
where \(\delta\) is a real number from \([0,1]\), and \(X\) represents a set of inputs, i.e., \(X := \{x_1, x_2, \ldots, x_n\}\). Since datasets often contain noisy data or challenging instances that the model cannot predict accurately, we introduce the parameter \(\delta\) to accommodate standard classification settings in which Type I and Type II errors are non-negligible. Intuitively, the introduction of $\delta$ allows multiple inputs $x_1, ..., x_n$ to vote on a neuron's state. For instance, when $\delta$ is set to 0.99, then 99\% or more of the inputs must agree that a neuron is activated for the neuron to be in the $\mathbf{1}$ state. It is worth mentioning that the statistical \(\widetilde{\mathcal{A}}\) function is equivalent to the method described in \citet{geng23}.

\begin{definition}[Class NAP]
In a classification task with the class set \(C\), for any class \(c \in C\), a class NAP \(P^c\) is a NAP comprising abstract states outputted by an abstraction function given \(X_c\), where \(X_c\) denotes the set of inputs belonging to class \(c\). Formally, \(P^c := \langle \mathcal{A}(N_{i, l}, X_c) \mid N_{i, l} \in N \rangle\), where $\mathcal{A}$ is either $\widetilde{\mathcal{A}}$ or $\dot{\mathcal{A}}$, chosen for each neuron. The power set of  \(P^c\) is denoted as $\mathcal{P}^c$.
\end{definition}

Robustness verification involves proving that no adversarial examples exist in the local neighborhood of a reference point \(x\). For NAP specifications, this translates to showing that no adversarial examples exist in class NAP \(P^c\); in other words, inputs exhibiting \(P^c\) must be classified as \(c\). \citet{geng23} argues that class NAPs must meet several key requirements to qualify as NAP specifications, and finding such NAPs effectively verifies the underlying robustness problem. We formalize these requirements into the following properties:

% , \text{ } \mathcal{A} \in \{ \ddot{\mathcal{A}}, \dot{\mathcal{A}} \}

\paragraph{The non-ambiguity property}
% \NL{Do another pass}
Since we want class NAPs to serve as certificates for a certain class, they must be distinct from each other. Otherwise, there could exist an input that exhibits two class NAPs, which leads to conflicting predictions. Formally, we aim to verify the following:
\begin{align}
    \forall x \quad \forall c_1, c_2 \in C \text{ s.t. } c_1 \neq c_2  \quad P^{c_1} \preccurlyeq  \ddot{\mathcal{A}}(N, x)    \Longrightarrow
    P^{c_2} \undernegpreccurlyeq \ddot{\mathcal{A}}(N, x)
\end{align}
From a geometric perspective, there must be no overlaps between class NAPs. In other words, it is also equivalent to verifying:
\[
\forall c_1, c_2  \in C \text{ s.t. } c_1 \neq c_2 \quad R_{P^{c_1}} \bigcap R_{P^{c_2}} = \emptyset
\]
\paragraph{The NAP robustness property} To serve as NAP specifications, class NAPs $P^c$ ensure that if an input exhibits it, i.e., $P  \preccurlyeq  \ddot{\mathcal{A}}(N,x)$, the input must be predicted as the corresponding class $c$. Formally, we have:
\begin{align}
    \forall x \in R_{P^{c}} \quad \forall k \in C \text{ s.t. } k\neq c \quad  F_c(x) - F_k(x) > 0
\end{align}
in which 
\begin{align}
    R_{P^{c}}= \{ x \mid {  P^c \preccurlyeq \ddot{\mathcal{A}}(N,x)}  \}
\end{align}
% In contrast to canonical $L_\infty$ norm balls, class NAPs are more flexible in terms of size and shape. Additionally, there is no need to specify a reference point, since the locations of potential reference points are also encoded by class NAPs. However, it is possible that no class NAP \(P^c\) in \(\mathcal{P}^c\) satisfies this property. This can be mitigated by meeting the subsequent weaker property.
% ,  \mathcal{A} \in \{ \ddot{\mathcal{A}}, \dot{\mathcal{A}} \}

\paragraph{The NAP-augmented robustness property} Rather than relying exclusively on class NAPs as specifications, local neighbors can also be used in combination for verification. This hybrid approach offers several advantages: 1) It narrows the scope of verifiable regions when class NAPs alone cannot satisfy the NAP robustness property; 2) NAP constraints essentially lock ReLU states, thereby refining the search space for verification tools; 3) It emphasizes verifying valid test inputs, while still covering broader and more flexible verifiable regions than those defined by $L_\infty$ norm ball specifications alone. Formally, this property can be expressed as:
\begin{align}
    \forall x' \in B(x, \epsilon) \bigcap   R_{P^{c}} \quad  \forall k \in C \text{ s.t. } k\neq c \quad   F_c(x) - F_k(x) > 0
\end{align}
in which 
\begin{align}
    B(x, \epsilon) = \{ x' \mid ||x-x'||_\infty \leq \epsilon \}  \quad R_{P^{c}}= \{ x' \mid { P^c \preccurlyeq \ddot{\mathcal{A}}(N,x')} \}
\end{align}
To summarize, we state that a class NAP can serve as a NAP specification if it satisfies either the NAP robustness property or the NAP-augmented robustness property. Clearly, the former property is stronger, and it is possible that we can't find a class NAP \(P^c\) in \(\mathcal{P}^c\) to satisfy this property. Fortunately, we can always find NAPs that satisfy the latter property by narrowing the verifiable region with additional $L_\infty$ norm ball specifications. 

%% file: method.tex
\section{The minimal NAP specification problem}

% \section{Problem Formulation and Naive Approaches}

\label{problemformulation}
In this section, we formulate the problem of learning minimal NAP specifications and present both a deterministic and statistical approach to address it. As our methods involve interactions with verification tools, we first introduce the relevant notations. 

% that describe the relationships between NAPs and these tools.

% We say a class NAP is a NAP specification is it  
% Recall section \ref{sec:NAP_property},
% As we discussed in Section \ref{sec:NAP_property}, The role of such tools is to verify the above robustness properties, which are expressed as inequalities in the form of \(\mathbf{F}_c(x) - \mathbf{F}_k(x) > 0\). We replace such inequalities with \(\mathbf{F}(x) > 0\) for a simpler presentation, assuming that \(\mathbf{F}(x)\) functions as a binary classifier. Thus, in the rest of the paper, whenever we mention verifying \(\mathcal{V}(P^c) > 0\), it implies a call to the verification tool.

\subsection{Problem Formulation}
Let $(\mathcal{P}^c, \preccurlyeq)$ be a partially ordered set corresponding to a family of class NAPs regarding some class $c \in C$. For simplicity, we omit the superscript \(c\) and refer to class NAPs simply as NAPs when the context is clear. We assume access to a verification tool, $\mathcal{V}: \mathcal{P} \to \{0, 1\}$, which maps a class NAP $P \in \mathcal{P}$ to a binary set. Here, $\mathcal{V}(P) = 1$ denotes a successful verification of the underlying robustness query, while 0 indicates the presence of an adversarial example. From an alternative perspective, $\mathcal{V}(P) = 1$ also signifies that $P$ is a NAP specification, i.e., it satisfies NAP(-augmented) robustness properties; whereas $\mathcal{V}(P) = 0$ implies the opposite.

It is not hard to see that $\mathcal{V}$ is monotone with respect to the NAP family $(\mathcal{P}, \preccurlyeq)$. Given $P' \preccurlyeq P$ and  $\mathcal{V}(P') = 1$, it follows that $\mathcal{V}(P) = 1$. However, given $P' \preccurlyeq P$ and $\mathcal{V}(P) = 1$, we cannot determine $\mathcal{V}(P')$. In other words, refining a NAP (by increasing the number of neurons abstracted to  \(\mathbf{0}\) or \(\mathbf{1}\)) can only enhance the likelihood of successful verification of the underlying robustness query. 

% We also define $|P|$ as the size of NAP $P$, which represents the number of neurons that are abstracted by $\ddot{\mathcal{A}}$ in $P$. Formally, $|P| := | \{ \ddot{\mathcal{A}}(N_{i, l}) | \ddot{\mathcal{A}}(N_{i, l}) \in P \}|$. 
% \begin{figure}[t]
%   \centering
% \includegraphics[width=0.8\textwidth]{figures/black_circles.pdf}
%   \caption{A DAG representing all possible NAPs, with green nodes signifying they can serve as a specification, while red nodes indicate they cannot. Within each node, a black circle means the corresponding neuron is abstracted by $\widetilde{\mathcal{A}}$, whereas the white circle means it is abstracted by $\dot{\mathcal{A}}$. Thus, the top node represents the coarsest NAP, while the bottom node entails the most refined one. Our goal is to find the NAP that resides in the green nodes while its size is minimal.
%   }
%   \label{fig:DAG}
% \end{figure}

% \[
% \underset{P \in \mathcal{P}, \mathcal{V}(P) = 1}{\arg\min} |P|
% \]
\begin{definition}[The minimal NAP specification problem]
Given a family of NAPs $\mathcal{P}$ and a verification tool $\mathcal{V}$, the minimal NAP specification problem is to find a NAP $P$ such that:
\[
\forall P' \preccurlyeq P, \, P' \neq P \implies \mathcal{V}(P') = 0
\]
That is, when $P$ is minimal, any NAP $P'$ that is strictly coarser than $P$ will result in $\mathcal{V}(P') = 0$. The size of $P$, denoted by \(|P|\) or $s$, represents the number of neurons abstracted to \(\mathbf{0}\) or \(\mathbf{1}\), i.e., $|\{ N_{i, l} \mid P_{i, l} = \mathbf{0} \text{ or } \mathbf{1} \}|$.
\end{definition}
% \begin{remark}
% We can also define the size of NAP $P$ as the number of neurons that abstracted by $\ddot{\mathcal{A}}$, denoted as $|P| := | \{ \ddot{\mathcal{A}}(N_{i, l}) | \ddot{\mathcal{A}}(N_{i, l}) \in P \}|$. Then the minimal condition is equivalent to finding the smallest $P$ that $\mathcal{V}(P) = 1$. Then the minimal NAP specification problem can be restated as:
% \[
% \underset{P \in \mathcal{P}, \mathcal{V}(P) = 1}{\arg\min} |P|
% \]

% \begin{enumerate}
%     \item $P$ is a correct NAP specification, i.e., $\mathcal{V}(P) = 1$; 
%     \item $P$ is minimal, meaning that for any NAP $P'$ that is more coarsen than $P$, $\mathcal{V}(P') = 0$. Formally, this can be expressed as: $\forall P \preccurlyeq P', P' \neq P : \mathcal{V}(P') = 0$.
% \end{enumerate}

% \end{remark}

It is important to note that "minimal" refers to the level of abstraction, not the size of the NAP. Since $(\mathcal{P}^c, \preccurlyeq)$ is a partially ordered set, multiple minimal NAP specifications may exist, rather than a single global minimum. In such cases, selecting any one of the minimal NAPs is sufficient. However, it is possible that even the most refined NAPs cannot verify the robustness query, in which case no minimal NAP specification exists. Given the computational cost of using verification tools, we are particularly interested in methods that efficiently find a minimal NAP specification, minimizing the number of calls to $\mathcal{V}$.

\subsection{Conservative Bottom Up Approach}

We introduce \textsc{Coarsen}, a method that starts with the most refined NAP and gradually coarsens it to learn minimal NAP specifications. Before proceeding, we define the relevant notations.

Given a class input \(X_c\), the coarsest NAP is defined as the one that uses \(\dot{\mathcal{A}}\) to abstract each neuron in \(N\), denoted as \(\dot{P} := \langle \dot{\mathcal{A}}(N_{i, l},X_c) \mid N_{i, l} \in N \rangle\). This NAP is the smallest possible in size, with \(|\dot{P}| = 0\). Conversely, the most refined NAP applies \(\widetilde{\mathcal{A}}\) to abstract each neuron, denoted as \(\widetilde{P} := \langle \widetilde{\mathcal{A}}(N_{i, l},X_c) \mid N_{i, l} \in N \rangle\) or $\widetilde{\mathcal{A}}(N)$. This NAP is the largest in size, with \(|\widetilde{P}| \leq |N|\). Clearly, we have \(\dot{P} \preccurlyeq \widetilde{P}\).

To refine any NAP \(P\) through a specific neuron \(N_{i, l}\), we apply the \(\widetilde{\mathcal{A}}\) function, denoting this refinement as \(\widetilde{\Delta}(N_{i, l})\). This action may either increase or leave \(|P|\) unchanged. Conversely, we denote the coarsen action as \(\dot{\Delta}(N_{i, l})\), which will either decrease or leave \(|P|\) unchanged.

The \textsc{Coarsen} approach starts by verifying if the most refined NAP \(\widetilde{P}\) successfully passes verification. For each neuron, we attempt to coarsen it using \(\dot{\mathcal{A}}\). If the resulting NAP no longer verifies the query, we revert to refining it back using \(\widetilde{\mathcal{A}}\); otherwise, we retain the coarsened NAP. The procedure is detailed in Algorithm \ref{alg:Coarsen}, which may require up to \(|N|\) calls to \(\mathcal{V}\) in the worst case, as shown in Theorem \ref{thm:theorem_simplecoarsen}. Please refer to the proof in Appendix \ref{appendix:proof_simple}.

\begin{algorithm2e}[H]
    \caption{ \textsc{Coarsen}}
    \label{alg:Coarsen}
    \small
    \DontPrintSemicolon
    \SetKwProg{Fn}{Function}{}{end}
    \SetKwFunction{Coarsen}{Coarsen}
    \KwIn{The neural network $N$ }
    \KwOut{A minimal NAP specification $P$}
    \Fn{Coarsen(N)}{
    $P \leftarrow  \widetilde{\mathcal{A}}(N)$ \\
    \uIf{$\mathcal{V} (P) == 0$}{
        \Return $None$ ; \tcc*{Return None if the most refined NAP fails verification}
    }
    \Else{
        \For{$N_{i, l}$ \textbf{in} $N$}{
            $P \leftarrow \dot{\Delta}(N_{i, l})$ ; \tcc*{Try to coarsen(abstract) this neuron using $\dot{\Delta}$}
            \uIf{$\mathcal{V} (P) == 0$}{
                $P \leftarrow \widetilde{\Delta}(N_{i, l})$ ; \tcc*{Refine the neuron back if the verification fails}
            }
        }
        \Return $P$;
    }}
\end{algorithm2e}

\begin{theorem}
\label{thm:theorem_simplecoarsen}
The algorithm \textsc{Coarsen} returns a minimal NAP specification with $\mathcal{O}(|N|)$ calls to $\mathcal{V}$.
\end{theorem}

\subsection{Statistical Coarsen Approach}

The \textsc{Coarsen} algorithm initiates with the most refined NAP and progressively coarsens each neuron until verification fails. Enhancing the algorithm's performance is possible by coarsening multiple neurons during each iteration. However, a fundamental question emerges: How do we determine which \textit{set} of neurons to coarsen in each round? 

We present \textsc{StochCoarsen} to answer this question. In this approach, we assume that each neuron is independent of the others and select neurons to coarsen in a statistical manner. Specifically, in each iteration, we randomly coarsen a subset of refined neurons in the current NAP simultaneously to see if the new NAP can pass verification. We repeat this process until the NAP size reaches $s$. This method can be regarded as optimistic bottom up. Algorithm \ref{alg:samplecoarsen} provides the pseudocode for \textsc{StochCoarsen}.

It's important to note that the stochastic manner of this algorithm faces the challenge of sample efficiency. For instance, if a minimal NAP specification \(P\) of size \(s\) is identified after one iteration, the probability of selecting the exact \(s\) essential neurons in \(P\) is \(\theta^{s}\). Consequently, if we set \(\theta\) as a constant, the expected number of samples required to find the NAP becomes \((\frac{1}{\theta})^{s}\).

\textsc{StochCoarsen} allows us to narrow down the estimated essential neurons by an expected factor of \(\theta\) once a valid NAP is learned. Thus, the expected number of samples and iterations are inversely related, with their product equating to the total calls to \(\mathcal{V}\). By setting \(\theta = e^{-\frac{1}{s}}\), we can achieve polynomial expected samples in \(s\) while minimizing the total number of calls to \(\mathcal{V}\), as proven in Theorem \ref{thm:theorem_coarsen}. The proof is detailed in Appendix \ref{appendix:proof_stats}.

\begin{algorithm2e}[H]
    \caption{\textsc{StochCoarsen}}
    \label{alg:samplecoarsen}
    \small
    \DontPrintSemicolon
    \SetKwProg{Fn}{Function}{}{end}
    \SetKwFunction{Coarsen}{Coarsen}
    \KwIn{The neural network $N$, the probability $\theta$, and the size $s$}
    \KwOut{A minimal NAP specification $P$}
    
    % \Fn{Gradient\_Heuristics($unvisited$, $\theta$)}{
    %     $P \leftarrow \dot{\mathcal{A}}(N)$ \tcc*{Use the coarsest NAP as a blank template}
    %     \For{$N_{i, l}$ \textbf{in} $unvisited$}{
    %         $rand \leftarrow \text{random}(0,1)$ \;
    %         \uIf{$rand \leq \theta$}{
    %             $P_{i,l} \leftarrow \widetilde{\mathcal{A}}(N_{i, l})$ \tcc*{Refine unvisited neurons using $\widetilde{\mathcal{A}}$ with probability $\theta$}
    %         }
    %     }
    %     \Return $P$\;
    % }
    \Fn{StochCoarsen($mand\_neurons$, $\theta$, $s$)}{
    $P \leftarrow \widetilde{\mathcal{A}}(N)$; $mand\_neurons \leftarrow N$\;
    \uIf{$\mathcal{V}(P) == 0$}{
        \Return $None$ \tcc*{Return None if the most refined NAP fails verification}
    }
    \Else{
        \While{$|P| > s$}{
            $P \leftarrow Sample\_NAPs(mand\_neurons , \theta)$\;
            \uIf{$\mathcal{V}(P) == 1$ }{ 
                $found\_neurons \leftarrow \emptyset$ \;
                \For{$N_{i, l}$ \textbf{in} $N$}{
                    \uIf{$P_{i, l} == \widetilde{\mathcal{A}}(N_{i, l})$} {
                        $found\_neurons  \leftarrow found\_neurons  \cup \{N_{i, l}\}$ 
                    }
                }$mand\_neurons \leftarrow found\_neurons$  \tcc*{Reduce search space}
            }
            \Else{
                $P \leftarrow Sample\_Naps(mand\_neurons , \theta)$ \tcc*{Sample a new NAP}
            }
        }
        \Return $P$ \tcc*{Return the minimal NAP of size $s$}
    }
    }
\end{algorithm2e}

\begin{theorem}
\label{thm:theorem_coarsen}
With probability $\theta$=$e^{-\frac{1}{s}}$, \textsc{StochCoarsen} learns a minimal NAP specification with $\mathcal{O}(slog|N|)$ calls to $\mathcal{V}$.
\end{theorem}

\paragraph{Setting the sample probability $\theta$} Setting $s$ poses a challenge in practice, as we assume that $s$ is always provided in \textsc{StochCoarsen}. However, this can be addressed by dynamically updating $\theta$ based on the result of $\mathcal{V}(P)$ \cite{minimal_abs}. With $\theta$ from theorem \ref{thm:theorem_coarsen}, \textsc{StochCoarsen} finds a NAP specification with probability $\left( e^{-1/s} \right)^s = e^{-1}$. Thus, we aim to set $\theta$ such that the $Pr(\mathcal{V}(P)=1) = e^{-1}$. Intuitively, if a sampled NAP $P$ is a specification, we decrease $\theta$ so more neurons will be coarsened. Similarly, if $P$ is not a specification, $\theta$ needs to be increased.

Given that $\theta \in [0,1]$, we can parameterize it using the Sigmoid function $\sigma(\lambda) = \left( 1 + e^{-\lambda} \right)^{-1}$, where $\lambda \in (-\infty,\infty)$. Since $Pr(\mathcal{V}(P)=1)$ depends on $\theta$ as well, we express it as a function of $\lambda$, \( g(\lambda) = Pr( \mathcal{V}(P) = 1) \). Then, setting $Pr(\mathcal{V}(P)=1) = e^{-1}$ can be achieved through the following minimization problem:

\begin{equation}
    L(\lambda) = \frac{1}{2}(g(\lambda) - e^{-1/s})^2
\end{equation}

The loss function $L(\lambda)$ can be minimized by statistical learning using stochastic gradient descent. With a step size \(\eta\), update $\lambda$ using $\lambda \leftarrow \lambda - \eta \frac{dL}{d\lambda} $. Note that $\frac{dL}{d\lambda}$ can be expressed as:

\begin{equation}
    {\frac{dL}{d\lambda}} = (g(\lambda) - e^{-1/s}) \frac{dg(\lambda)}{d\lambda}
\end{equation}

Given \( g(\lambda) = Pr( \mathcal{V}(P) = 1) \), we can replace $g(\lambda)$ with $\mathcal{V}(P)$ for stochastic gradient update. Additionally, since $\frac{dg(\lambda)}{d\lambda} > 0$, we simply ignore it as its multiplication effect can be represented by $\eta$.  Therefore, the final update rule is given by:
\begin{equation}
\lambda \leftarrow \lambda - \eta(\mathcal{V}(P) - e^{-1})
\end{equation}

\section{Optimistic Approach}
While \textsc{Coarsen} and \textsc{StochCoarsen} effectively identify minimal NAP specifications with correctness guarantees, their practicality is limited by the extensive runtime required to perform a large number of expensive verification calls. This limitation makes them less suitable for time-critical applications or large-scale neural networks. To address these challenges, we propose a optimistic approach for efficiently initializing NAP specifications. This approach significantly reduces the number of subsequent verification calls and provide a reliable upper bound on the size of the minimal NAP specification, improving both efficiency and scalability. The approach focuses on the concept of essential neurons, key to understanding the minimal NAP specification problem.

\begin{definition}[Essential neuron]
A neuron $N_{i,l} \in N$ is considered essential if it cannot be coarsened to $\mathbf{*}$ in any minimal NAP specification. We denote the set of all essential neurons as $E$, defined by:
\[
E = \{ N_{i,l} \mid P_{i,l} \in \{\mathbf{0}, \mathbf{1}\}, P \textit{ is minimal} \}
\]
Note that $E$ is the union of the set of essential neurons from all minimal NAP specifications. It follows that $|E| \geq s$, where $s$ denotes the size of the largest minimal NAP specification.
\end{definition}

The minimal NAP specification problem can be solved trivially if we gain access to $E$. Thus, our optimistic approach is designed to determine essential neurons and estimate $E$. To better understand our approach, we first discuss the properties of essential neurons.
Recall that verifying a robustness query given a NAP specification $P$ is equivalent to showing that $\mathbf{F}(x) \geq 0$ for input $x$ in region $R_P$. Thus, the necessary conditions of a essential neuron $N_{i,l}$ can be written as follows:

\begin{enumerate}
    \item If $N_{i, l}$ is in state $\mathbf{0}$,  it implies when $ \hat{z}_i^{(l)}(x) = 0, \mathbf{F}(x) \geq 0$. In addition, $\exists x \text{ s.t. } \hat{z}_i^{(l)}(x) > 0, \mathbf{F}(x) < 0$.
    \item If $N_{i, l}$ is in state $\mathbf{1}$, it implies $\forall x \text{ s.t. } \hat{z}_i^{(l)}(x) > 0, \mathbf{F}(x) \geq 0$. In addition, when  $ \hat{z}_i^{(l)}(x) = 0, \mathbf{F}(x) < 0 $. 
\end{enumerate}

\begin{figure}[b]
    \centering
     \begin{subfigure}[t]{0.23\textwidth}
         \centering         \includegraphics[width=\textwidth]{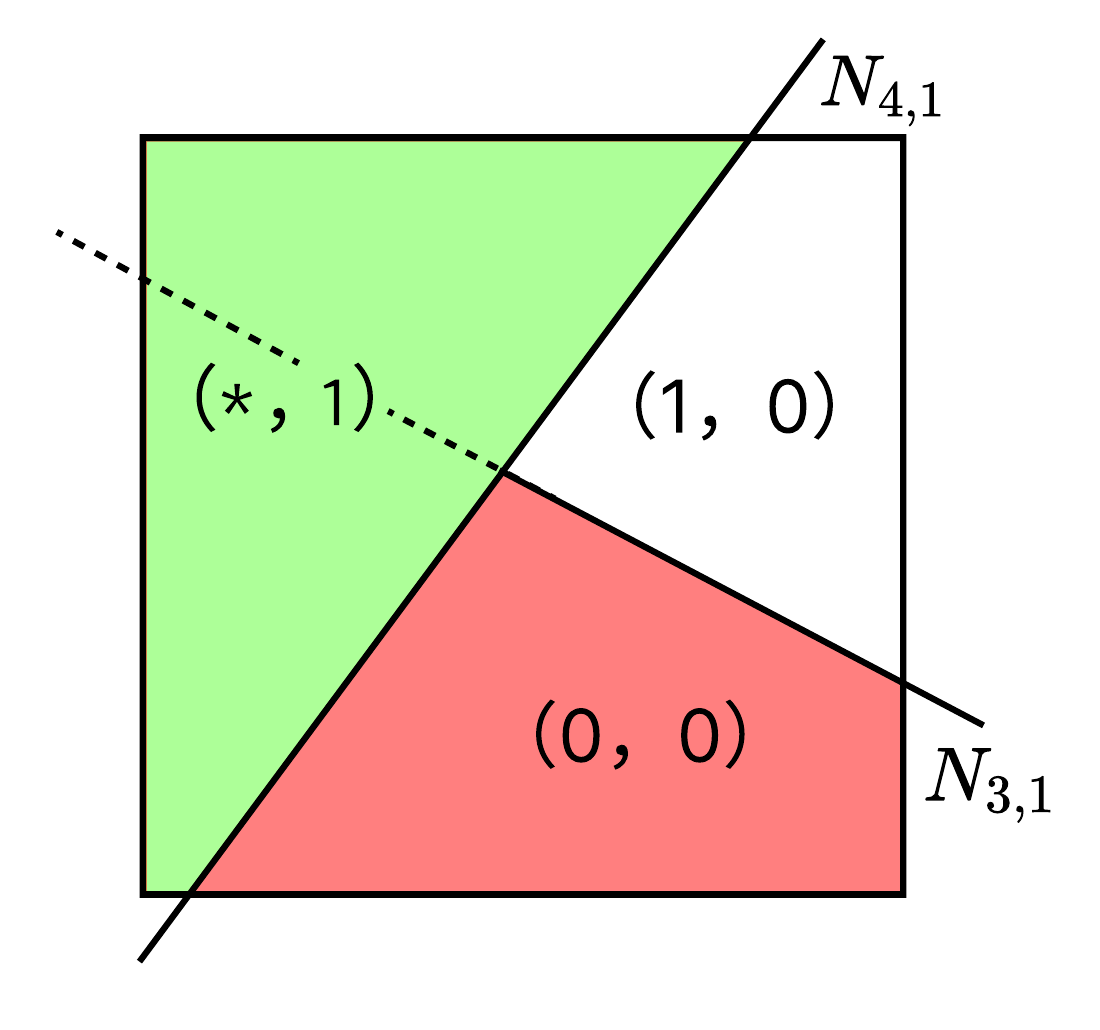}
         \caption{$N_{4,1}$ is essential.}
         \label{fig:adv_region1}
         \hfill
     \end{subfigure}
     \begin{subfigure}[t]{0.23\textwidth}
         \centering         \includegraphics[width=\textwidth]{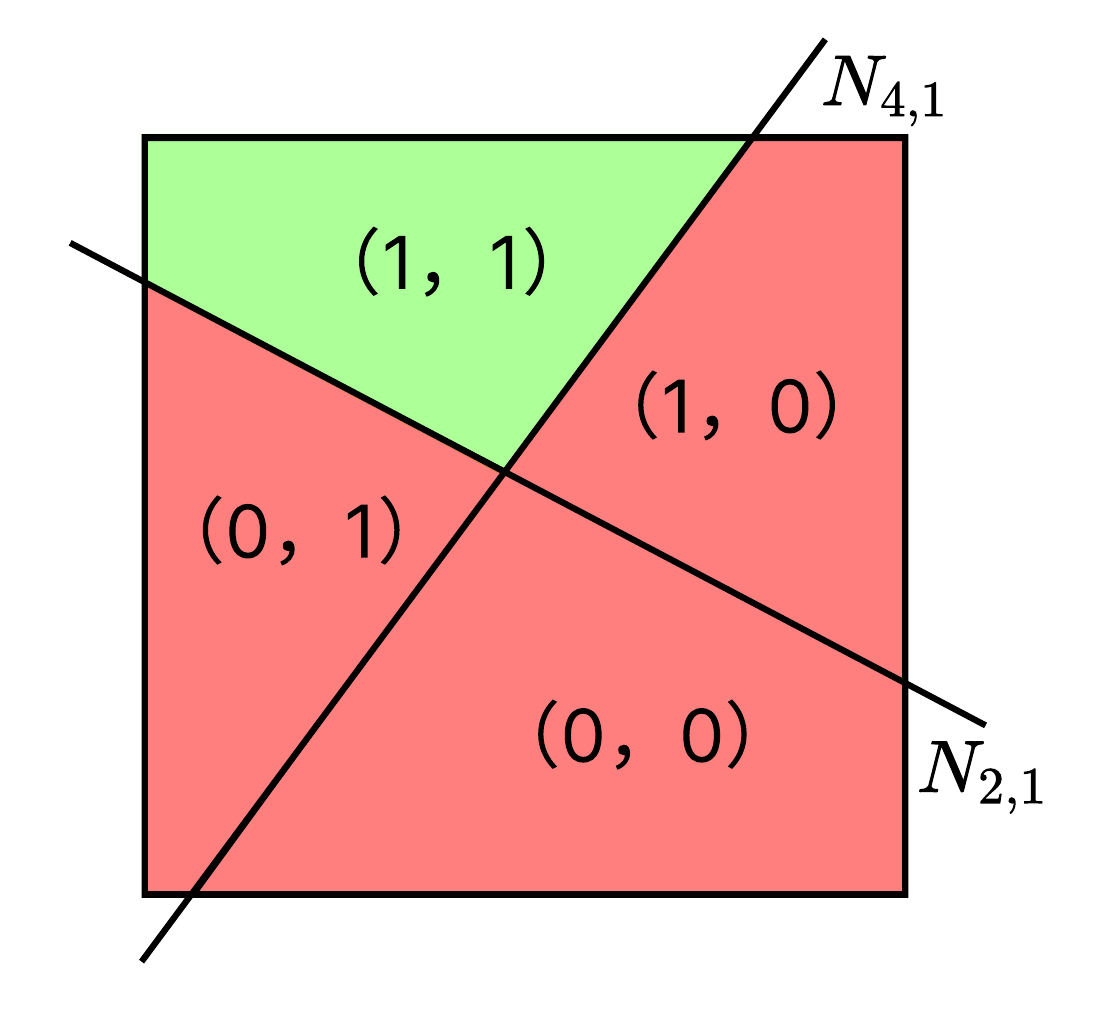}
         \caption{$N_{4,1},N_{2,1}$ are essential.}
         \label{fig:adv_region5}
         \hfill
     \end{subfigure}
     \begin{subfigure}[t]{0.23\textwidth}
         \centering         \includegraphics[width=\textwidth]{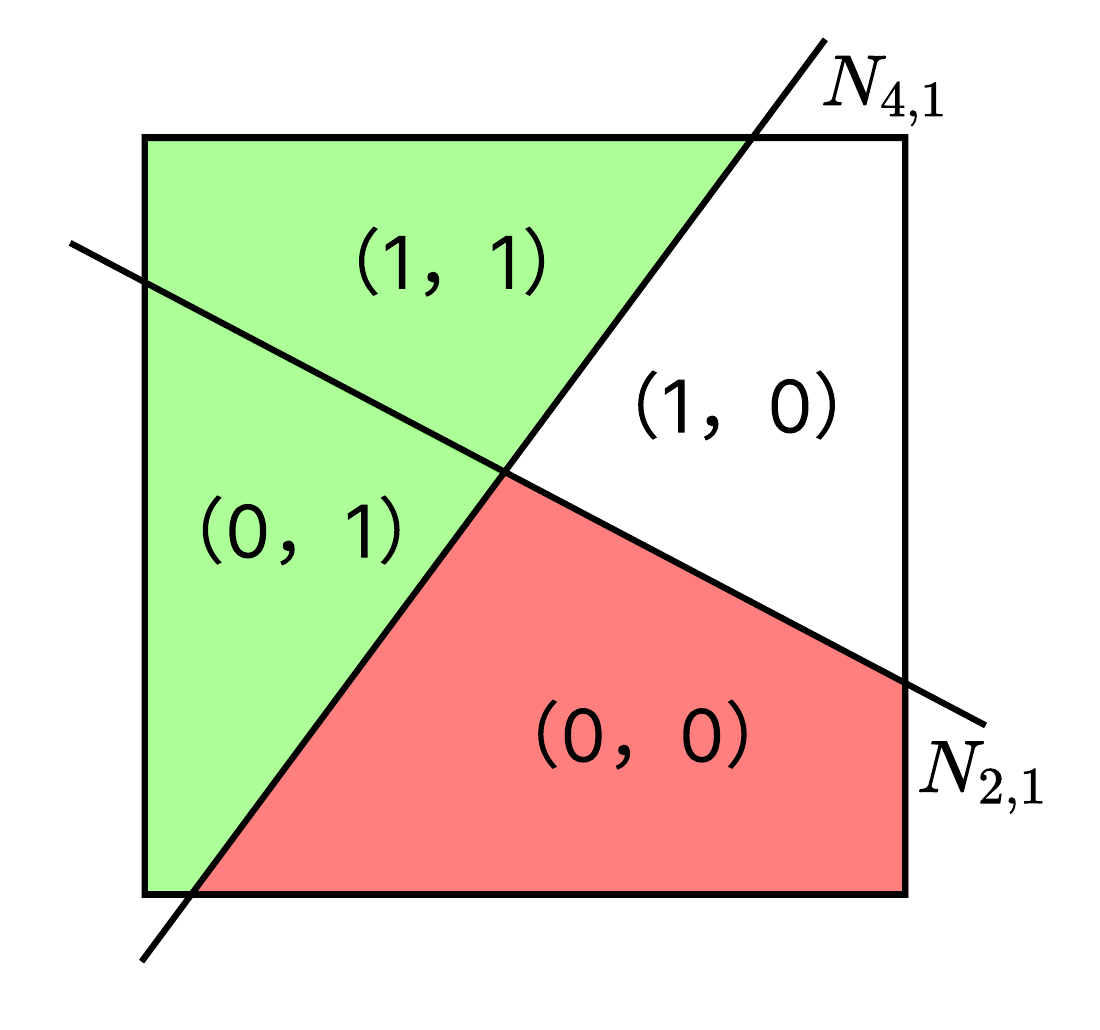}
         \caption{$N_{4,1}$ is essential.}
         \label{fig:adv_region3}
         \hfill
     \end{subfigure}
     \begin{subfigure}[t]{0.23\textwidth}
         \centering         \includegraphics[width=\textwidth]{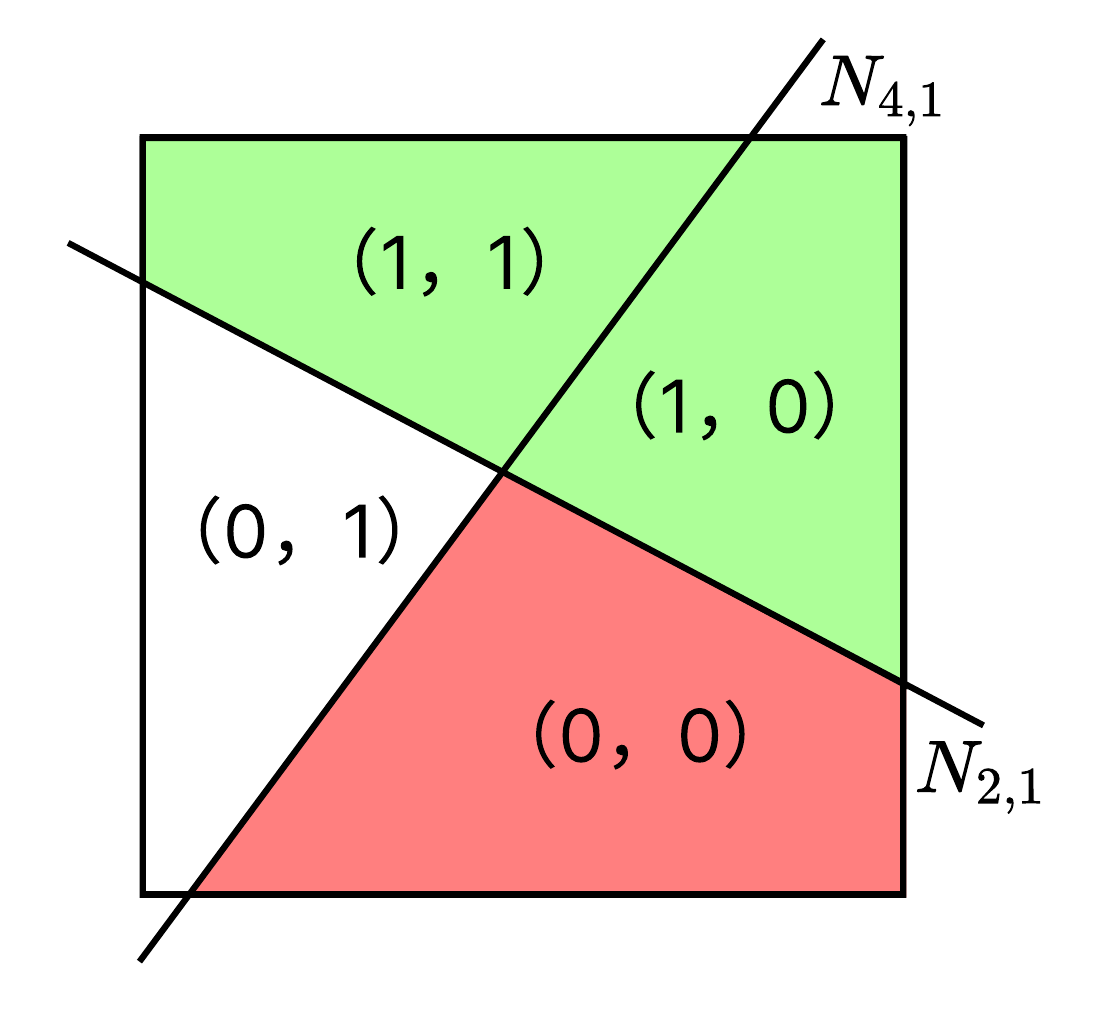}
         \caption{$N_{2,1}$ is essential.}
         \label{fig:adv_region2}
         \hfill
     \end{subfigure}
        \caption{Geometric interpretation of NAPs on essential neurons. The first subfigure represents the case when the two NAPs disagree on $N_{4,1}$. The other three subfigures represent the three cases when the two NAPs disagree on $N_{4,1},N_{2,1}$. Regions colored green pass verification, whereas red indicates an adversarial example exists.
        Here, we omit states for $N_{1,1},N_{3,1}$ in those NAPs for simplicity.}
        \label{fig: Geometric_interpretation_essentia_neurons}
        % \vspace{-50pt}
\end{figure}

As for non-essential neurons in $P$, since they can be coarsened to $\mathbf{*}$, it implies $\mathbf{F}(x) \geq 0$ regardless of the value of $\hat{z}_i^{(l)}(x)$. Formally, this can be written as: If $N_{i, l}$ is in state $\mathbf{0}$ or $\mathbf{1}$, it implies that $\forall x, \mathbf{F}(x) \geq 0$. In the \textsc{Coarsen} algorithm, we rely on interaction with the verification tool $\mathcal{V}$ to identify essential neurons. Since calls to $\mathcal{V}$ are typically computationally expensive, it would be advantageous to estimate $E$ in a more cost-effective manner. This motivates us to study the following optimistic approach.

We introduce \textsc{OptAdvPrune} to identify essential neurons. Intuitively, it attempts to show a neuron $N_{i,l}$ is essential by actively falsifying NAP candidates with adversarial examples. When an adversarial example $x'$ is found, it immediately indicates that the NAP $\ddot{\mathcal{A}}(N, x' )$ fails the verification, i.e., $\mathcal{V}(\ddot{\mathcal{A}}(N, x' )) = 0$. Moreover, it also implies that any NAP subsumed by $\ddot{\mathcal{A}}(N, x' )$ fails verification. For instance, suppose an adversarial example $x'$ is found for a simple one-layer four-neuron neural network and $\ddot{\mathcal{A}}(N, x' )$ is $\langle\mathbf{1},\mathbf{0},\mathbf{1},\mathbf{0}\rangle$. We can infer that NAPs like $\langle\mathbf{1},\mathbf{0},\mathbf{1},\mathbf{*}\rangle$, $\langle\mathbf{1},\mathbf{0},\mathbf{*},\mathbf{*}\rangle$, $\langle\mathbf{1},\mathbf{*},\mathbf{*},\mathbf{0}\rangle$, and $\langle\mathbf{1},\mathbf{0},\mathbf{*},\mathbf{0}\rangle$ fail the verification. This information is particularly useful when determining if a neuron is essential. For example, if we know that NAP $P := \langle\mathbf{1},\mathbf{0},\mathbf{*},\mathbf{1}\rangle$ is a specification, i.e., $\mathcal{V}(P) = 1$, then we can easily deduce that the fourth neuron $N_{4,1}$ is essential. This is because that coarsening the fourth neuron would expand $P$ to $\langle\mathbf{1},\mathbf{0},\mathbf{*},\mathbf{*}\rangle$, which would include the adversarial example $x'$ and thus fail the verification, as illustrated in Figure \ref{fig:adv_region1}. It is evident that the neuron where $P$ and $\ddot{\mathcal{A}}(N, x' )$ disagrees must be essential. %\rebecca{Given a simple one layer four neuron neural network with NAP P:= 1,0,1,* such that V(P)=1. Suppose we find a adversarial example x' with A(N, x')=1,0,1,0. Observe that coarsening the fourth neuron would include the adversarial example, making V(P)=0.  }

\begin{algorithm2e}[t]
    \caption{\textsc{OptAdvPrune}}  
    \label{alg:Pruning}
    \small
    \DontPrintSemicolon
    \SetKwProg{Fn}{Function}{}{end}
    \SetKwFunction{Coarsen}{Coarsen}
    \KwIn{The neural network $N$, the input dataset $X$ }
    \KwOut{A collection of essential neurons}
    
    \Fn{OptAdvPrune($N$, $X$)}{
        $Mand \leftarrow \emptyset$ ; $P \leftarrow \widetilde{\mathcal{A}}(N) $ \; 
            \For{$x_{j}$ \textbf{in} $X$}{ 
             $x'_{j} \leftarrow Adversarial\_Attack(x_j)$ \;
             \For{$N_{i,l}$ \textbf{in} $N$}{
              \uIf{$P_{i, l} \in \{\mathbf{0},\mathbf{1}\}$ and $\ddot{\mathcal{A}}(N_{i,l}, x'_j ) \oplus P_{i,l} $} {
                 $Mand \leftarrow Mand \cup \{N_{i, l}\}$   \tcc*{Create essential neurons for $x'_{j}$}
             }}

        }
        \Return $Mand$\;
        }      
\end{algorithm2e}

However, when the two NAPs disagree on multiple neurons, things become a little bit different. Suppose the NAP specification $P$ is $\langle\mathbf{1},\mathbf{1},\mathbf{*},\mathbf{1}\rangle$, i.e., $\mathcal{V}(\langle\mathbf{1},\mathbf{1},\mathbf{*},\mathbf{1}\rangle) = 1$. We know $\mathcal{V}(\langle\mathbf{1},\mathbf{0},\mathbf{*},\mathbf{0}\rangle) = 0$ by the adversarial example $x'$. In this case, if we coarsen the second and fourth neurons, $N_{2,1}$ and $N_{4,1}$, the NAP specification will expand to $\langle\mathbf{1},\mathbf{*},\mathbf{*},\mathbf{*}\rangle$, which will cover the $\langle\mathbf{1},\mathbf{0},\mathbf{*},\mathbf{0}\rangle$, thus failing the verification. In this case, $N_{2,1}$ and $N_{4,1}$ could both be essential neurons or either one of them is essential, as illustrated in Figures \ref{fig:adv_region5}, \ref{fig:adv_region3}, \ref{fig:adv_region2}. So, we simply let $\{N_{2,1}, N_{4,1}\}$ be the upper bound of essential neurons (learned from $x'$). Formally, given a NAP $P$, we say a neuron $N_{i,l}$ is in the upper bound of essential neurons $E$ if satisfies the following condition:

\begin{enumerate}
\item $N_{i,l}$ must be in the binary states, i.e., $P_{i,l} \in \{\mathbf{0}, \mathbf{1}\}$
\item There exists $x'$ such that $\ddot{\mathcal{A}}(N_{i,l}, x' )$ XORs with $P_{i,l}$, i.e., $\exists x' \text{ such that } \ddot{\mathcal{A}}(N_{i,l}, x' ) \oplus P_{i,l} = 1$
\end{enumerate}

The field of adversarial attacks provides a diverse set of computationally efficient methods, enabling access to numerous adversarial examples. In this study, we use a collection of different attacks, including the Projected Gradient Descent (PGD) attack \cite{pgd} and the Carlini-Wagner (CW) attack \cite{cw2}. By computing the upper bound of essential neurons for each adversarial example and taking their union, we efficiently estimate the overall upper bound, as shown in Algorithm \ref{alg:Pruning}.

\section{Volume Estimation of $R_{P}$}
\label{appendix: nap volume}

% \section{Calculating(approximating) verifiable volumn}
% \textbf{Working with NAPs using Marabou} In this paper, we use Marabou \citep{Marabou}, a dedicated state-of-the-art NN verifier. Marabou extends the Simplex \cite{simplex} algorithm for solving linear programming with special mechanisms to handle non-linear activation functions. Internally, Marabou encodes both the verification problem and the adversarial attacks as a system of linear constraints (the weighted sum and the properties) and non-linear constraints (the activation functions). Same as Simplex, at each iteration, Marabou tries to fix a variable so that it doesn't violate its constraints. While in Simplex, a violation can only happen due to a variable becoming out-of-bound, in Marabou a violation can also happen when a variable doesn't satisfy its activation constraints.

% NAPs and NAP properties can be encoded using Marabou with little to no changes to Marabou itself. To force a neuron to be activated or deactivated, we add a constraint for its output. To improve performance, we infer ReLU's phases implied by the NAPs, and change the corresponding constraints\footnote{Marabou has a similar optimization, but the user cannot control when or if it is applied.}.For example, given a ReLU $v_i=max(v_k, 0)$, to enforce $v_k$ to be activated, we remove the constraint from Marabou and add two new ones: $v_i = v_k$, and $v_k \geq 0$.
\begin{wrapfigure}{R}{0.4\textwidth}
  \centering
  \includegraphics[width=0.4\textwidth]{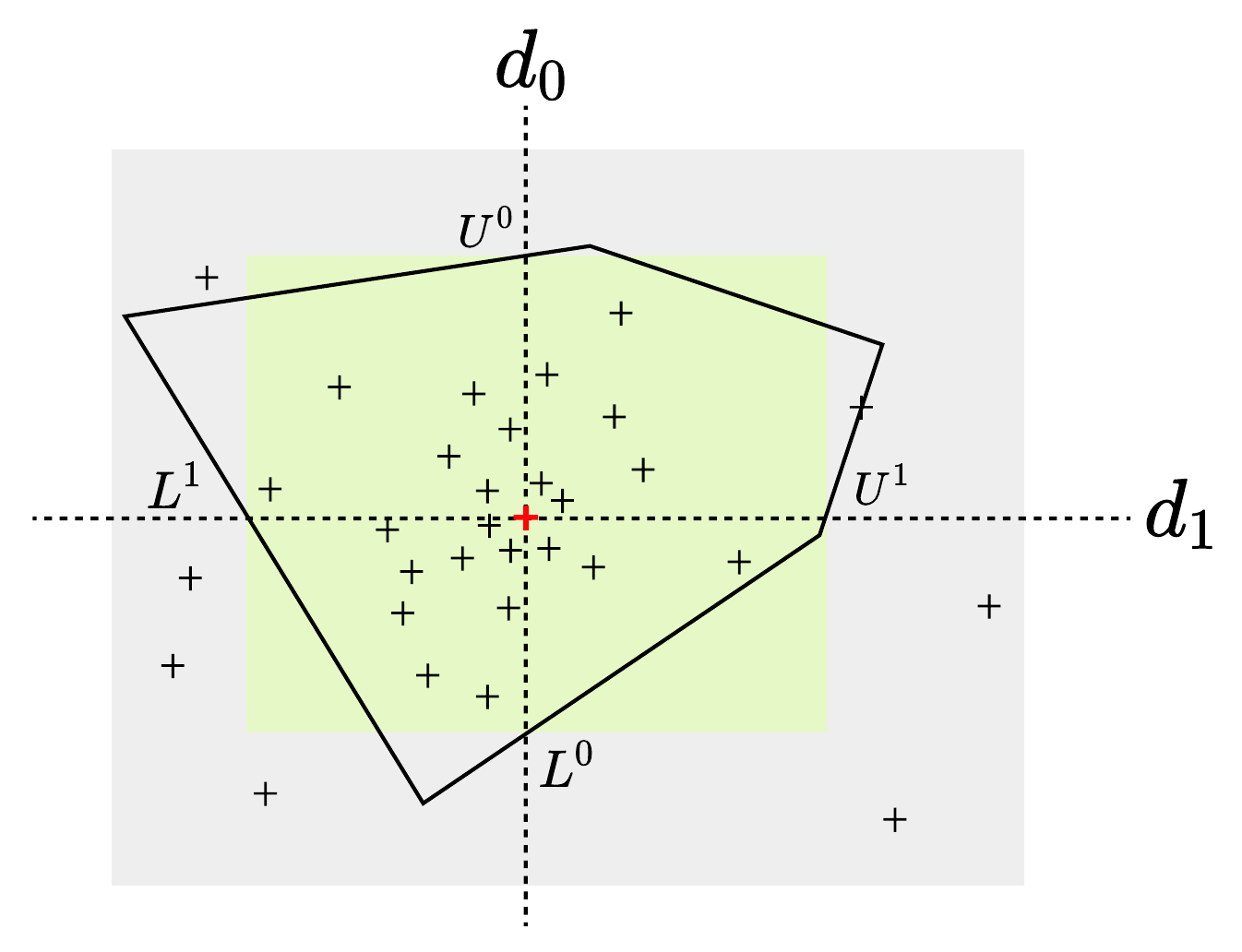}
  \caption{Volume Estimation of \( R_{P} \) using an orthotope in a 2-dimensional case. The gray rectangle represents the input space, with the training set depicted by a collection of data points \( + \). The polygon corresponds to some NAP \( P \). Initially, we identify an anchor point, denoted by \textcolor{red}{\( + \)}. Then, we construct the orthotope, represented by the green rectangle, by extending upper and lower bounds starting from the anchor point until it extends beyond \( P \).}
  \label{fig:volume}
\end{wrapfigure}

Conceptually, NAP specifications typically correspond to significantly larger input regions compared to local neighbourhood specifications. This serves as the primary motivation for utilizing NAPs as specifications. However, previous work lacks sufficient justification or evidence to support this claim. In this section, we propose a simple method for approximating the volume of $R_{P}$, i.e., the region corresponding to a NAP $P$. This allows us to: 1) quantify the size difference between $R_{P}$ and $L_\infty$ ball specifications; 2) gain insights into the volumetric change from the most refined NAP specification to the minimal NAP specification.

Computing the exact volume of $R_{P}$ is at least NP-hard, as determining the exact volume of a polygon is known to be NP-hard \cite{volume_np}. Moreover, computing the exact volume of $R_{P}$ can be even more challenging due to its potential concavity. To this end, our method estimates the volume of $R_{P}$ by efficient computation of an orthotope that closely aligns with $R_{P}$, as illustrated in Figure \ref{fig:volume}. We briefly describe it as follows:

\paragraph{Finding an anchor point} The first step is to find an anchor point to serve as the center of the orthotope. Ideally, this anchor point should be positioned close to the center of $R_{P}$ to ensure a significant overlap between the orthotope and $R_{P}$. However, computing the actual center of $R_{P}$ is costly. Thus, we look for a pseudo-center from the training set $X$ that resides in $R_{P}$. This pseudo-center can be computed by finding the point that uses the smallest $L_{\infty}$ ball to cover other data points, solved as the following optimization problem:
\[
c_{\text{pseudo}} = \underset{x \in R_{P}}{\arg\min} \max_{x' \in  R_{P}} \|x - x'\|_{\infty}
\]
where $R_{P} = \{ x \mid \mathcal{A}(N,x) \preccurlyeq P, x \in X \}$. When $|X|$ is small, $c_{\text{pseudo}}$ can be computed directly; for larger $|X|$, a statistical computation strategy is required.

\paragraph{Constructing the orthotope} Once the pseudo-center \( c_{\text{pseudo}} \) is determined, we want to create an orthotope around \( c_{\text{pseudo}} \) to closely align with \( R_{P} \). The orthotope is constructed by determining pairs of upper and lower bounds \( U^{(i)} \) and \( L^{(i)} \) for each dimension \( i \). Specifically, \( U^{(i)} \) and \( L^{(i)} \) are computed through expansion in two opposite directions from \( c_{\text{pseudo}} \) along dimension \( i \) until they extend beyond \( R_{P} \). This expansion can be expressed as:

\[
\max_{U^{(i)}} \{ x' \in R_{P} \,|\, x' := c_{\text{pseudo}} + U^{(i)} \}
\text{ ; }
\max_{L^{(i)}} \{ x' \in R_{P} \,|\, x' := c_{\text{pseudo}} - L^{(i)} \}
\]

Here, \( U^{(i)} \) and \( L^{(i)} \) represent the upper and lower bounds in dimension \( i \) respectively, originating from \( c_{\text{pseudo}} \). These bounds can be efficiently calculated with binary search.

The choice of the archer point is crucial in our approach. If it is located at a corner of \( R_p \), the volume calculation will be highly biased. This can pose a problem when we seek to understand the volumetric change from the most refined NAP specification to the minimal NAP specification. Additionally, using the orthotope as an estimator provides convenience in understanding the volumetric change simply by examining differences in each input dimension.

%% file: evaluation.tex
\section{Evaluation}
\label{sec:eval}

In this section, we conduct a comprehensive evaluation of our algorithms for learning minimal NAP specifications across a range of benchmarks, from decision-critical tasks in cancer diagnosis to state-of-the-art image classification models. To illustrate the effectiveness of our approaches, we chose the method proposed in \cite{geng23} as the baseline, denoted as the $\widetilde{\mathcal{A}}$ function. 

\paragraph{Experiment Setup}
All experiments in this section were conducted on an Ubuntu 20.04 LTS machine with 172 GB of RAM and an Intel(R) Xeon(R) Silver Processor. For verification, we utilized Marabou \cite{Marabou}, a dedicated state-of-the-art neural network verifier. We set a timeout of 5 minutes for each query call to the verification tool. If the timeout is exceeded, the current neuron is retained in the minimal NAP specification even if its status cannot be determined.
\subsection{The Wisconsin Breast Cancer Dataset with Binary Classifier} 
We conduct our first experiment using a four-layer neural network as a binary classifier, where each layer consists of 32 neurons. This classifier is trained on the Wisconsin Breast Cancer (WBC) dataset \cite{wbc}. Our trained model achieves a test set accuracy of 95.61\%. We calculate the most refined (baseline) NAP specifications $\widetilde{P}^{0}$ and $\widetilde{P}^{1}$ for labels 0 and 1 using the statistical abstraction function $\widetilde{\mathcal{A}}$ with a confidence ratio of $\delta = 0.95$. The size of $\widetilde{P}^{0}$ and $\widetilde{P}^{1}$ are  102 and 93, respectively.
In contrast, the sizes of the minimal NAP specifications learned by the \textsc{Coarsen} algorithm for labels 0 and 1 are significantly reduced to 31 and 32, respectively. 
It is worth mentioning that our optimistic approach provides a fairly accurate estimate of the essential neurons, despite computing rather loose upper bounds. To be more specific, \textsc{OptAdvPrune} compute 43  essential neurons for label 0, respectively, covering 25 out of the 31  essential neurons appearing in the minimal NAP specification for label 0 computed by \textsc{Coarsen}. For label 1, \textsc{OptAdvPrune} computes 39  essential neurons, covering 25 out of the 32  essential neurons appearing in the minimal NAP specification for label 1 computed by \textsc{Coarsen}. 

The NAP specification learnt from \textsc{OptAdvPrune} are verified, showing it provides an optimistic upper bound to the size of minimal NAP with minimal verification calls yet finding neurons . We show it's efficiency on small models. This sheds light to solving minimal NAP problem on large models when using Coarsen and  \textsc{StochCoarsen} are impossible given the current limit on verification engine. Regarding the statistical approaches, \textsc{StochCoarsen} learnt NAP specifications of size 42 and 45 for label 0 and label 1 using only 47 and 41 calls to verification respectively. 

Recall that one of the main motivations for learning the minimal NAP specifications is their potential to verify larger input regions compared to refined NAP specifications. To support this, we compute the percentile of unseen test data that can be verified using these NAPs. Test data, sampled from the input space, serve as a proxy to understand the verifiable bounds of different NAPs. We find that the most refined NAP specifications $\widetilde{P}^{0}$ and $\widetilde{P}^{1}$ cover 81.40\% and 80.28\% of test data for labels 0 and 1, respectively. In contrast, minimal NAP specifications cover 95.35\% and 94.37\% of test data for labels 0 and 1, respectively. To intuitively understand the change in verifiable regions \( R_P \) from refined to minimal NAP specifications, we compare their estimated volumes. The increase in estimated volume is substantial—on the order of \( 10^5 \) times larger for labels 0 and 1.

% 8324.35 times larger for label 0 and 30503.53 times larger for label 1. 
% 58324.35

\begin{table}[t]
\centering
\caption{Overview of the size of learned minimal NAP specifications using various approaches on the WBC benchmark. The columns represent different labels, while each row corresponds to a different algorithm. \text{|$P$|} denotes the size of the learned NAP, and \text{\#$\mathcal{V}$} represents the number of calls to $\mathcal{V}$. The \textit{train} and \textit{test} columns report the percentile (\%) of train and test data covered by $P$. We assess the effectiveness of NAP specifications in covering both training and test data, reporting the coverage percentage (\%) along with the estimated change in volume, expressed as an order of magnitude (x) relative to a baseline normalized to 1.}
\label{tab:mnist_res}
\begin{tabular}{|l|ccccc|ccccc|} 
\hline 
& \multicolumn{5}{c|}{0} & \multicolumn{5}{c|}{1} \\
\cline{2-11}
& \text{|$P$|} & \text{\#$\mathcal{V}$} & \textit{train} & \textit{test} & \textit{$vol.$} & \text{|$P$|} & \text{\#$\mathcal{V}$} & \textit{train} & \textit{test} & \textit{$vol.$} \\
\hline \hline
{\fontsize{7}{12}\selectfont$\widetilde{\mathcal{A}}$} \normalsize{function} (\textit{baseline.}) & 102 & 1 & 78.11 & 81.40 & 1 & 93 & 1 & 83.06 & 80.28 & 1 \\
$\textsc{Coarsen}$ & 31 & 102 & 98.22 & 95.35 & $10^5$ & 32 & 93 & 99.65 & 94.37 & $10^5$ \\
$\textsc{StochCoarsen}$ & 42 & 47 & 94.69 & 92.34 & $10^3$ & 45 & 41 & 94.15 & 91.52 & $10^2$ \\
$\textsc{OptAdvPrune}$ & 61 & 1 & 91.06 & 89.37 & $10^2$ & 54 & 3 & 87.06 & 87.28 & $10^2$ \\
\hline
\end{tabular}
\end{table}

% \begin{figure}[h]
%     \centering
%     \begin{subfigure}{0.49\textwidth}
%         \centering
%         \includegraphics[width=\textwidth]{figures/range_compare_label0.pdf} % Adjust the width as needed
%         \caption{Comparing verifiable input ranges for label 0}
%         \label{fig:wbc_refined}
%     \end{subfigure}
%     \hfill
%     \begin{subfigure}{0.49\textwidth}
%         \centering
%         \includegraphics[width=\textwidth]{figures/range_compare_label1.pdf} % Adjust the width as needed
%         \caption{Comparing verifiable input ranges for label 1}
%         \label{fig:wbc_minimal}
%     \end{subfigure}
%     \caption{Comparison of verifiable input ranges between the refined NAP specifications and the minimal NAP specifications for labels 0 and 1 on the Wisconsin Breast Cancer dataset, with reference to the anchor point. The minimal NAP specifications enable verification over a larger input range for all dimensions.}
%     \label{fig:wbc_range_comparison}
% \end{figure}

% \textbf{MNIST with Fully Connected Network} 
% \paragraph{MNIST with Fully Connected Network} 

\subsection{MNIST with Fully Connected Network}
To show that our insights and approaches can be applied to more complicated datasets and networks, we conduct the second set of experiments using the \texttt{mnistfc\_256x4} model \cite{vnncomp2021},  a 4-layer fully connected network with 256 neurons per layer trained on the MNIST dataset. Due to space constraints, we present results for randomly selected classes 0, 1, and 4, with the remaining results provided in Appendix \ref{appendix:mnist}.

% \begin{table}[b]
%     \centering
% \caption{Overview of learned minimal NAP specifications on the MNIST benchmark.}
%     \label{tab:mnist_res}
%     \resizebox{\textwidth}{!}{\begin{tabular}{ |l | c c c c c | c c c c c | c c c c c| }
%         \hline
%         & \multicolumn{5}{c|}{0} &  \multicolumn{5}{c|}{1}  & \multicolumn{5}{c|}{4} \\
%         \cline{2-16}
%         & \text{|$P$|} &  \text{\#$\mathcal{V}$}  & \textit{train$\%$} & \textit{test$\%$}& \textit{$vol.$} & \text{|$P$|}& \text{\#$\mathcal{V}$}& \text{train$\%$}& \text{test$\%$}  & \textit{$vol.$} & \text{|$P$|}& \text{\#$\mathcal{V}$}& \text{train$\%$}& \text{test$\%$}  & \textit{$vol.$} \\
%         \hline\hline
%         {\fontsize{7}{12}\selectfont$\widetilde{\mathcal{A}}$} \normalsize{function} (\textit{baseline.})  & 751 & 1 & 79.50 & 71.51 & 1 & 745 & 1 & 75.01 & 70.11 & 1 & 712 & 1 & 77.54 & 75.24 & 1 \\ 
%         $\textsc{Coarsen}$ & 480 & 751 & 98.68& 98.78 & 10^8 & 491 & 745 & 98.90 & 98.59 & 10^6 & 506 & 712 & 98.51 & 97.45 & 10^9 \\ 
%         $\textsc{StochCoarsen}$ & 532 & 33 & 93.19 & 93.01 & 10^5 & 559 & 27 & 94.12 & 93.68 & 10^3 & 562 & 25 & 93.89 & 93.64 & 10^6 \\ 
%         $\textsc{OptAdvPrune}$ & 618 & 15 & 83.13 & 81.32 & 10^{2} & 630 & 18 & 86.02 & 85.51 & 10^{2} & 699 & 21 & 82.66 & 84.12 & 10^2 \\ 
%         \hline
%     \end{tabular}}
% \end{table}

\begin{table}[b]
    \centering
    \caption{Overview of learned minimal NAP specifications on the MNIST benchmark.}
    \label{tab:mnist_res}
    \resizebox{\textwidth}{!}{%
        \begin{tabular}{ |l | c c c c c | c c c c c | c c c c c| }
            \hline
            & \multicolumn{5}{c|}{0} &  \multicolumn{5}{c|}{1}  & \multicolumn{5}{c|}{4} \\
            \cline{2-16}
            & $|P|$ &  $\#\mathcal{V}$  & \textit{train\%} & \textit{test\%} & \textit{vol.} & $|P|$ & $\#\mathcal{V}$ & \textit{train\%} & \textit{test\%} & \textit{vol.} & $|P|$ & $\#\mathcal{V}$ & \textit{train\%} & \textit{test\%} & \textit{vol.} \\
            \hline\hline
            {\fontsize{7}{12}\selectfont$\widetilde{\mathcal{A}}$} \normalsize{function} (\textit{baseline.})  & 751 & 1 & 79.50 & 71.51 & 1 & 745 & 1 & 75.01 & 70.11 & 1 & 712 & 1 & 77.54 & 75.24 & 1 \\ 
            $\textsc{Coarsen}$ & 480 & 751 & 98.68 & 98.78 & $10^8$ & 491 & 745 & 98.90 & 98.59 & $10^6$ & 506 & 712 & 98.51 & 97.45 & $10^9$ \\ 
            $\textsc{StochCoarsen}$ & 532 & 33 & 93.19 & 93.01 & $10^5$ & 559 & 27 & 94.12 & 93.68 & $10^3$ & 562 & 25 & 93.89 & 93.64 & $10^6$ \\ 
            $\textsc{OptAdvPrune}$ & 618 & 15 & 83.13 & 81.32 & $10^{2}$ & 630 & 18 & 86.02 & 85.51 & $10^{2}$ & 699 & 21 & 82.66 & 84.12 & $10^2$ \\ 
            \hline
        \end{tabular}%
    }
\end{table}

We compute the most refined NAP specifications, $\widetilde{P}^{0}$, $\widetilde{P}^{1}$, and $\widetilde{P}^{4}$, for these labels with a confidence level of $\delta = 0.99$. Their sizes are 751, 745, and 712, respectively, consistent with the baseline results from previous work \cite{geng23}. In contrast, the minimal NAP specifications learned by the \textsc{Coarsen} algorithm for labels 0, 1, and 4 are significantly reduced to 480, 491, and 506, respectively, despite requiring over 700 calls. These specifications cover more than 98\% of the training and test data, demonstrating a significant improvement over the baseline. For estimated volume, the volumetric changes can be on the order of magnitude of \(10^9\) larger. Notably, \textsc{StochCoarsen} achieves comparable results with only around 30 calls, at the cost of approximately a 5\% reduction in coverage on both training and test data compared to \textsc{Coarsen}. This highlights the potential of these learned minimal NAPs as robust specifications that generalize well to unseen data from the same distribution. In comparison, our empirical analysis reveals that local neighborhood specifications cover \emph{zero} test images in the entire MNIST dataset with \(\epsilon = 0.2\), which is the maximum \(L_\infty\) verifiable bound used in VNNCOMP \cite{vnncomp2023}, the annual neural network verification competition. This further validates our motivation for using NAPs as specifications.

% as local neighborhood specifications has two major limitations: 1) it primarily covers a convex region of input data, which can be mathematically described by adding noise to the reference point, as illustrated in Figure \ref{fig:nap_intro}b; 2) it is too restrictive to cover unseen test set data, which are real data sampled from the underlying distribution.

% – the annual neural network verification competition – are usually less than 0.2, while the smallest distance between data points with the same label exceeds 0.5.

For instance, for label 0, \textsc{OptAdvPrune} detects 618 essential neurons and correctly identify 445 and 160 of the 480 neurons in the minimal NAP founded by \textsc{Coarsen}. When interacting with \(\mathcal{V}\), \textsc{OptAdvPrune} computes minimal NAPs with an average of fewer than 20 calls. However, it experiences a decline in data coverage, with an average drop of over 5\% compared to the baseline.

% On the other hand, our approxtmate approaches provide quite accurate estimates on  essential neurons(minimal NAPs) on the first attempt, without needing calls to the verification tool. For instance, for label 0, \textsc{Adversarial\_Prune} and \textsc{Gradient\_Search} detect 618 and 195  essential neurons, respectively, and correctly identify 445 and 160 of the 480 neurons in the minimal NAP specification. When interacting with $\mathcal{V}$, \textsc{Adversarial\_Prune} computes minimal NAPs with an average of fewer than 20 calls. However, it experiences a decline in data coverage, with an average drop of over 5\% compared to the baseline.

% Moreover, the learned minimal NAP specifications correspond to significantly larger verifiable regions compared to the refined NAP specifications. Using the percentile of test data as a metric, the minimal NAP specifications increase the coverage ratio from 80.51\% to 98.78\%, 85.11\% to 98.59\%, and 80.24\% to 97.45\% for labels 0, 1, and 4, respectively. 

In summary, there is no universal solution to a given minimal specification problem. Our experiment highlights the distinct strengths and trade-offs of the three algorithms. While the results from all have gurantee on correctness, \textsc{Coarsen} finds the minimal NAP among the three, making it the most reliable for applications where rigorous minimality is priority. \textsc{StochCoarsen} offers a significant speed advantage on top of \textsc{Coarsen} while maintaining minimality:  the learned minimal NAP specification uses approximately 1.1x more neurons but requires only 5\% of the verification calls. This significant reduction in computational overhead makes \textsc{StochCoarsen} particularly well-suited for scenarios where efficiency is a priority. Finally, \textsc{OptAdvPrune}, employing an optimistic approach, is the fastest. The initial set of neurons, when they can be successfully verified, provides a tight upper bound for the global minimum. When the initial set cannot be verified, it provides a effective starting point for \textsc{StochCoarsen} and \textsc{Coarsen}. Moreover, \textsc{OptAdvPrune} can be utilized in a verification-dependent manner, they excel in scenarios that scale beyond our verification capabilities. In these situations, they provide accurate estimations of essential neurons, enabling us to examine potential causal links between neurons and the interpretability and robustness of deep neural networks, as we will illustrate in the next experiment. Together, these algorithms apply to diverse scenarios, allowing users to balance speed and minimality accuracy, and guarantee of validity based on their specific requirements.

\subsection{ImageNet with Deep Convolutional Neural Network}

In our third experiment, we present our finds on the fully connected layers \footnote{NAP computation for convolutional layers is left for future work.} of VGG-19 network pretrained on the ImageNet dataset, consisting of 8192 neurons. Provided it is currently impossible to verify NAP specifications on models of this size, we focus on analyzing the effect of essential neurons estimated using \textsc{OptAdvPrune}. We also limit our analysis to the top five largest classes, each containing approximately 1000 training and 350 test images. Our findings indicate that these estimated NAPs cover a significant portion of unseen test data, highlighting their potential as robust certificates for the test set.  More details can be found in Appendix \ref{appendix:ImageNet}.

\paragraph{NAP Captures Visual Interpretability and Inherent Robustness}
From the perspective of representation learning, neural networks acquire both low- and high-level feature extractors, which they use to make final classification decisions based on hidden features (neuron representations) \cite{rep_learning}. Therefore, the robustness and consistency of a model's predictions are influenced by the quality of these learned features. In essence, achieving an accurate and robust model hinges on learning "good" hidden representations, which are characterized by better interpretability \cite{inter_survey}. Many studies suggest a close relationship between visual interpretability and robustness, often observed in the learned features and representations \cite{alvarez2018towards,boopathy2020proper,dong2017towards}. Thus, although we cannot yet formally verify the correctness of these estimated NAP specifications, we demonstrate that these NAPs are indeed "meaningful" through visual interpretability—strong evidence that the estimated  essential neurons (NAPs) contribute to the model's robustness.

\begin{figure}[h]
    \centering
    % Row 1
    \begin{minipage}[b]{\textwidth}
        \includegraphics[width=\linewidth]{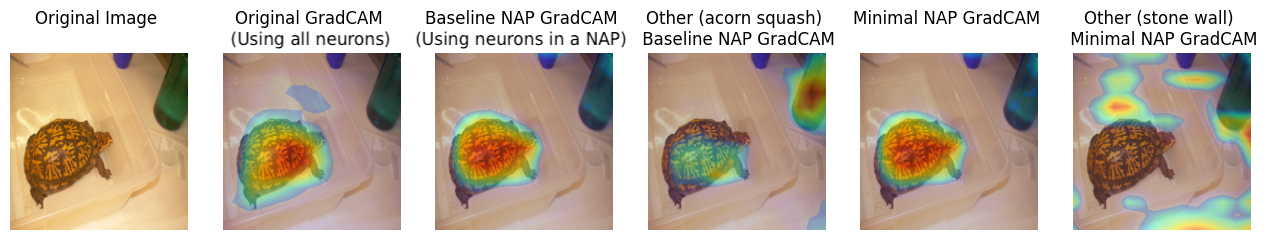}
        %\caption*{Box Turtle}
    \end{minipage}

        \begin{minipage}[b]{\textwidth}
        \includegraphics[width=\linewidth]{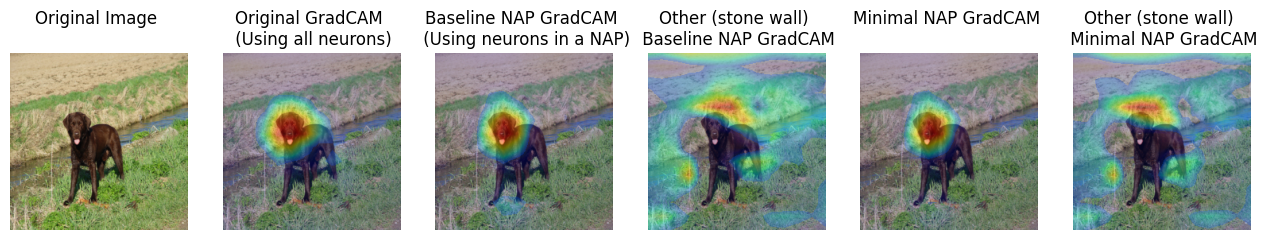}
        %\caption*{Box Turtle}
    \end{minipage}
    
    % \vspace{0.5cm} % Add vertical space between rows
    
    % Row 2
  %   \begin{minipage}[b]{\textwidth}
  % \includegraphics[width=\linewidth]{figures/spiny_lobster_.png}
  %       %\caption*{Spiny Lobster}
  %   \end{minipage}

  %       \begin{minipage}[b]{\textwidth}
  % \includegraphics[width=\linewidth]{figures/spiny_lobster2.png}
  %       %\caption*{Spiny Lobster}
  %   \end{minipage}

    % \vspace{0.5cm} % Add vertical space between rows
    
    % Row 2
    \begin{minipage}[b]{\textwidth}
  \includegraphics[width=\linewidth]{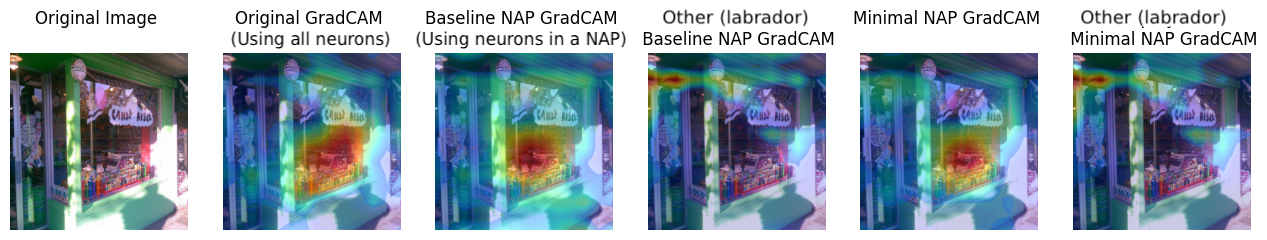}
        %\caption*{Confectionery}
    \end{minipage}

        \begin{minipage}[b]{\textwidth}
  \includegraphics[width=\linewidth]{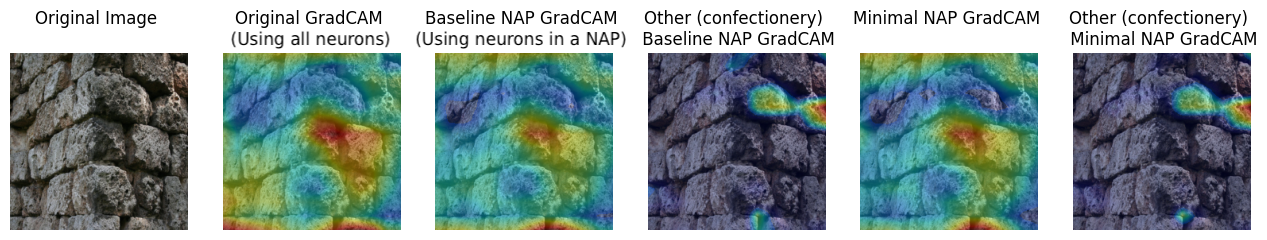}
        %\caption*{Stone Wall}
    \end{minipage}

  %   \begin{minipage}[b]{\textwidth}
  % \includegraphics[width=\linewidth]{figures/acorn_squash.png}
  %       %\caption*{Acorn Squash}
  %   \end{minipage}

    \begin{minipage}[b]{\textwidth}
  \includegraphics[width=\linewidth]{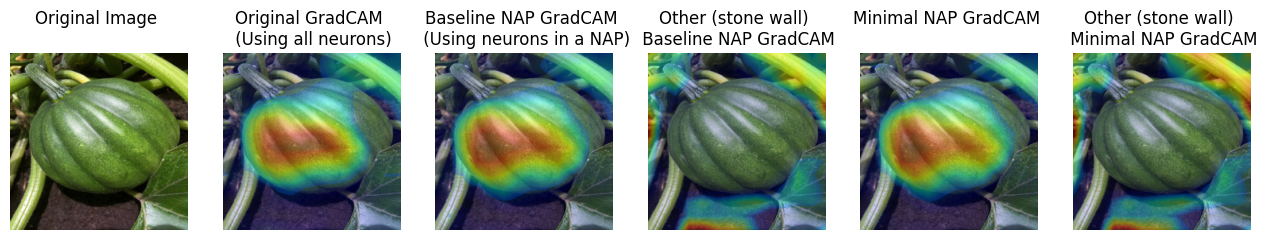}
        %\caption*{Acorn Squash}
    \end{minipage}
    
    \caption{Visual interpretability of learned hidden representations retained by estimated NAPs. The first two columns represent the original images and GradCams, respectively. The third and fourth columns represent modified GradCams using baseline NAPs of the same class and another class, respectively. The last two columns represent modified GradCams using minimal NAPs of the same class and another class, respectively.}
    \label{fig:gradcam}
\end{figure}

% we mask out the neurons that do not appear in NAP \( P \) using a mask \( M_{P} \) in the fully connected layers. We then calculate the backward gradient flow based on the modified computation graph, denoted as \( \frac{\partial y_c}{\partial A_{k,ij}} M_{P} \). By replacing \( \frac{\partial y_c}{\partial A_{k,ij}} \) with \( \frac{\partial y_c}{\partial A_{k,ij}} M_{P} \), we compute the modified Grad-CAM map.

% To this end, we employ Grad-CAM \cite{grad_cam}, a popular approach from the model interpretability domain. This technique leverages the gradients of the classification score with respect to the final convolutional feature map to highlight the most important regions of an input image. 

% To this end, we perform a simple modification to Grad-CAM \cite{grad_cam}, which leverages the gradients of the classification score with respect to the final convolutional feature map to highlight the most important regions of an input image,  
% we mask out the neurons that do not appear in NAP in the fully connected layers. We then calculate the backward gradient flow based on the modified computation graph, and  we compute the modified Grad-CAM map.

To this end, we make a simple modification to Grad-CAM \cite{grad_cam}, which highlights key regions of an input image by using the gradients of the classification score with respect to the final convolutional feature map. Specifically, we mask out the neurons that do not appear in NAP in the fully connected layers. We then recalculate the backward gradient flow using this modified computation graph to generate the updated Grad-CAM map. Finally, we conduct the following experiments on image samples:
\begin{enumerate}
    \item Calculate the modified Grad-CAM map using the most refined (baseline) NAP and compare it with the original Grad-CAM map.
    \item Calculate the modified Grad-CAM map using minimal NAPs estimated by \(\textsc{OptAdvPrune}\), and compare it with the original map.
    \item Calculate the modified Grad-CAM map using NAPs from different classes and compare it with the original map.
\end{enumerate}

All GradCam maps are further validated with sanity checks \citep{adebayo2018sanity}. Figure \ref{fig:gradcam} presents the experimental results.  The original GradCams highlight important regions of images corresponding to crucial justifications for classification. Notably, both the most refined and estimated minimal NAPs produce nearly identical highlights, despite representing only a small fraction of neurons.  This shows that the neuron abstractions from our minimal specification preserve the key visual features. Further, it suggests that the estimated minimal NAPs capture crucial aspects of VGG-19's internal decision-making process, supporting the "NAP robustness property." Furthermore, Grad-CAMs generated from NAPs of different classes highlight distinct regions, strongly indicating that our estimated NAPs are clearly distinguishable, aligning with the "non-ambiguity property."

These results demonstrate that NAPs offer valuable insights into interpretability. A small subset of neurons from NAPs can categorize critical internal dynamics of neural networks, potentially helping us unveil the black-box nature of these systems. From a machine-checkable definition perspective, concise NAPs are easier to decode into human-understandable programs than more refined NAPs. This underscores the importance of learning minimal NAPs. Interpreting NAPs into human-readable formats remains a direction for future research.

% This study demonstrates that NAPs hold significant value in interpretability. A small subset of neurons from NAPs can categorize critical internal dynamics of neural networks, potentially helping us unveil the black-box nature of these systems. From a machine-checkable definition perspective, concise NAPs are easier to decode into human-understandable programs than more refined NAPs. This underscores the importance of learning minimal NAPs. Interpreting NAPs into human-readable formats remains a direction for future research.

% From a machine-checkable definition perspective, if we can decode NAPs into human-understandable programs, concise NAPs will always be easier to decipher than the more refined NAPs. This also addresses the importance of the main research question — learning minimal NAPs. Interpreting NAPs into human-understandable formats remains a direction for future research. 
\paragraph{NAP as a Defense Against Adversarial Attacks}
From a practical point of view, we believe that even before 1) formal verification finally scales; and/or 2) NAPs are fully interpretable, NAPs, as they are now, can already serve as an empirical certificate of a prediction or some defense mechanism, as shown in recent work \cite{runtime_moniter}. In the same spirit, we demonstrate that our estimated essential neurons can serve as defense against adversarial attacks. 
% \rebecca{not confident about this wording: "serve as defense against adversarial attacks". Should we say something like: "Recall that the robustness verification can be thought of as proving that no adversarial examples exist in the input space constraint by the NAP. "? } 

% Essentially, we want images that follows our NAP to have consistent behavior by the model. So if the model 

% We first select images that meet two criteria: 1) correctly predicted by the model; and 2) covered by the respective NAP. For each image, we perform 100 times Projected Gradient Descent attack \cite{pgd} and 100 times Carlini Wagner attack \cite{cw2}, respectively.  We then check if the activation pattern of each generated attack image can be rejected by the respective NAPs. If not, we conclude that the NAPs can empirically serve as specifications. Otherwise, it is proven not to be robust. Notbaly,we find that both the baseline NAPs and the estimated minimal NAPs
% reject all adversarial examples, suggesting
% is effective in defending all generated adversarial attacks. 

We first select images that meet two criteria: 1) correctly predicted by the model; and 2) covered by the respective NAP. On average, each NAP covers approximately 40\% of the training data of the corresponding class. For each selected image, we generate 100 distinct adversarial examples that are misclassified by the model, using Projected Gradient Descent attack \cite{pgd} and Carlini Wagner attack \cite{cw2} respectively. We then check whether each adversarial image's activation pattern is rejected by the respective NAPs. If so, we conclude that NAPs can empirically serve as specifications; Otherwise, they are proven to be not robust. Notably, we find that both the baseline NAPs and the estimated minimal NAPs reject all adversarial examples, indicating their effectiveness in describing a safe region and their potential as certificates.

%% file: related.tex
\section{Related Work}
\label{sec:related}

\subsection{Neural Network Verification} 

Neural network verification has attracted much attention due to the increasingly widespread applications of neural networks in safety-critical systems. it's NP-hard nature resulting from the non-convexity introduced by activation functions \cite{reluplex} makes it a challenging task. Thus most of the existing work on neural network verification focuses on designing scalable verification algorithms. For instance, while initially proposed solver-based approaches \cite{smt,smt1,milp,milp1} were limited to verify small neural networks with fewer than 100 neurons, state-of-the-art methods \cite{ac,bc,bab} can verify more complex neural networks. It is worth mentioning that most existing work adopts local neighborhood specifications to verify the robustness properties of neural networks \cite{few_pixel}. Despite being a reliable measure, using local neighborhood specifications around reference data points may not cover any test data, let alone generalizing to the verification of unseen test set data. \citet{geng23} propose the new paradigm of NAP specifications to address this challenge. Our work advances the understanding of NAP specifications.
%Yet evidently, being able to verify the neural network on test data is an important task in real-world applications. This leads to the development of NAP specification \cite{geng23}. 
% Yet evidently, being able to verify the neural network on unseen test data is an important task in real-world applications.

% The development of NAP specification \cite{geng23} sheds light on this challenge. This work 

% However, it
% often focuses on computing the most refined NAPs, which include many redundant neurons. In this paper, we focus on computing minimal (general) NAP specifications. This not only allow us study robustness from the neurons' perspective but also significantly significantly expand the verifiable boundaries.

 % From the perspective of abstract interpretation, the NAP specification can be naturally viewed as an abstraction on neurons, where each neuron is abstracted to either binary states $\mathbf{0}$, $\mathbf{1}$, or the unary state $\mathbf{*}$.
\subsection{Abstract Interpretation} 

Abstract interpretation \citep{Absint} is a fundamental concept in software analysis and verification, particularly for approximating the semantics of discrete program states. By sacrificing precision, abstract interpretation typically enables scalable and faster proof finding during verification \cite{CousotC14}. Although abstract interpretation for neural network verification has been proposed and studied in previous literature \cite{AI2,eth}, abstract interpretation of neural activation patterns for verification is a relatively new field. Perhaps the most related work from the perspective of abstract interpretation is learning minimal abstractions \cite{minimal_abs}. While our work shares similarities in problem formulation and statistical approaches, we address fundamentally different problems. One limitation in our work is that our abstraction states may be too coarse: value in range $(0,+\infty)$ is abstracted into one state. This approach could over-approximate neuron behavior and thus fail to prove certain properties. We observe that neuron values exhibit different patterns in range for different input classes, suggesting the potential existence of more abstraction states. We leave this as future work.

\subsection{Neural Activation Patterns} 
% \NL{Please keep this part sort. We are already over the page limit.}
% Explaining why a neural network  comes up with a prediction has been a challenging problem, especially as models get increasingly complicated. 
% \allen{I drop this sentence. Now we can fit it in 9 pages}
Neural activation patterns have commonly been used to understand the internal decision-making process of neural networks. One popular line of research is feature visualization \cite{vis_network, NAPs1}, which investigates which neurons are activated or deactivated given different inputs. This is also naturally related to the field of activation maximization \cite{SimonyanVZ13}, which studies what kind of inputs could mostly activate certain neurons in the neural network. In this way, certain prediction outcomes may be attributed to the behavior of specific neurons, thereby increasing the interpretability of the underlying models. \citet{runtime_moniter} demonstrates that neural activation patterns can be used to monitor neural networks and detect novel or unknown input classes at runtime. They provide human-level interpretability of neural network decision-making.
In summary, most of existing works focus on learning statistical correlations between NAPs and inputs \cite{netdis, Erhan2009VisualizingHF}, or between NAPs and prediction outcomes \cite{runtime_moniter}. However, these correlations raises questions which we address in this paper: whether the correlation can be trusted or even verified. We propose the concept of essential neurons and highlight their importance in the robustness of model predictions. Such causal links between neurons and prediction outcomes are not only identified but also verified. We believe this "identify then verify" paradigm can be extended to existing research on NAPs to certify our understanding of neural networks. We leave the exploration of this direction for our future work.

%% file: conclusion.tex
\section{Conclusion}
\label{sec:conclusion}

We introduce a new challenge: learning the minimal NAP specification and highlighting its significance in neural network verification. Identifying minimal NAP specifications not only facilitates the verification of larger input regions compared to existing methods but also provides insights into when and how neural networks make reliable and robust predictions. To address this problem, we propose three approaches—conservative, statistical, and optimistic—each offering distinct trade-offs between efficiency and performance. The first two rely on the verification tool to find minimal NAP specifications. The optimistic method efficiently estimates minimal NAPs using adversarial examples, without making calls to the verification tool until the very end. Each of these methods offers distinct strengths and trade-offs in terms of minimality and computational speed, making each approach suitable for scenarios with different priorities. The learnt minimal NAP specification allows us to inspect potential causal links between neurons and the robustness of state-of-the-art neural networks, a task for which existing work fails to scale. Our experimental results suggest that minimal NAP specifications require much smaller fractions of neurons compared to the NAP specifications computed by previous work, yet they can significantly expand the verifiable boundaries to several orders of magnitude larger.

%% file: appendix.tex
\appendix
% \section{Appendix}
\label{Appendix}

\section{Learned Minimal NAP Specifications on the MNIST Benchmark (Complement)}
\label{appendix:mnist}

\begin{figure}[!h]
         \centering         
         \includegraphics[width=0.8\textwidth]{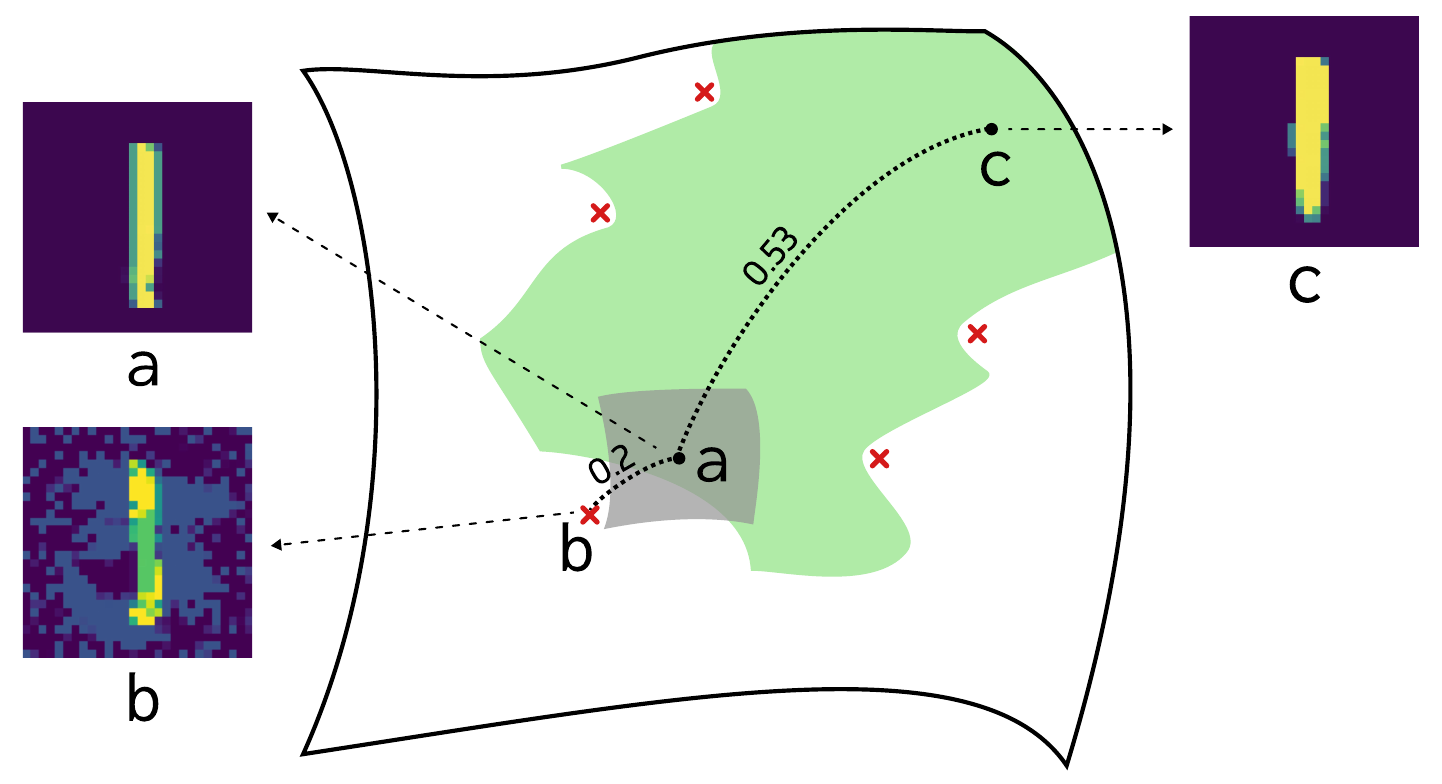}
         \caption{The comparison between NAP specifications (green region) and \(L_{\infty}\) ball specifications (gray region) on the MNIST dataset. Image \(a\) is the reference image, and image \(c\) is the closest among all 6000 training images of digit 1, with an \(L_{\infty}\) distance of 0.5294. However, \(c\) cannot be verified using the \(L_{\infty}\) ball specification, as an adversarial example \(b\) exists at an \(L_{\infty}\) distance of 0.2. Note that this is not a limitation of the underlying verification engines but rather an intrinsic limitation of the specifications. In contrast, NAP specifications allow verification of unseen test set data like \(c\).}
         \label{fig:nap_intro}
         \vspace{-10pt}
\end{figure}

\begin{table}[!h]
    \centering
\caption{Overview of the size of learned minimal NAP specifications on the MNIST benchmark.}
    \label{tab:mnist_res}
    \resizebox{\textwidth}{!}{\begin{tabular}{ |l | c c c c | c c  c c | c c c c| c c c c| }
        \hline
        & \multicolumn{4}{c|}{2} &  \multicolumn{4}{c|}{3}  & \multicolumn{4}{c|}{5} & \multicolumn{4}{c|}{6} \\
        \cline{2-17}
        & \text{$|P|$} &  \text{\#$\mathcal{V}$}  & \text{train} & \text{test}& \text{$|P|$}& \text{\#$\mathcal{V}$}& \text{train}& \text{test} & \text{$|P|$}& \text{\#$\mathcal{V}$}& \text{train}& \text{test} & \text{$|P|$}& \text{\#$\mathcal{V}$}& \text{train}& \text{test} \\
        \hline\hline
        Baseline \citep{geng23}  & 751 & 1 & 79.50 & 80.51 & 745 & 1 & 86.01 & 85.11 & 712 & 1 & 77.54 & 80.24  1 & 712 & 1 & 77.54 & 80.24\\ 
        $\textsc{Coarsen}$ & 480 & 751 & 98.68& 98.78 & 491 & 745 & 98.90 & 98.59 & 506 & 712 & 98.51 & 97.45 & 503 & 708 & 98.81 & 98.39 \\
        $\textsc{StochCoarsen}$ & 532 & 33 & 93.19 & 93.01 & 559 & 27 & 94.12 & 93.68 & 562 & 25 & 93.89 & 93.64 & 542 & 29 & 98.81 & 98.39\\ 
        $\textsc{OptAdvPrune}$ & 618 & 15 & 83.13 & 71.32 & 630 & 18 & 79.02 & 76.51 & 699 & 21 & 82.66 & 84.12 & 681 & 20 & 85.98 & 87.63 \\ 
        \hline
    \end{tabular}}
\end{table}

\begin{table}[t]
    \centering
\caption{Overview of the size of learned minimal NAP specifications on the MNIST benchmark.}
    \label{tab:mnist_res}
    \resizebox{\textwidth}{!}{\begin{tabular}{ |l | c c c c | c c  c c | c c c c| }
        \hline
        & \multicolumn{4}{c|}{7} &  \multicolumn{4}{c|}{8}  & \multicolumn{4}{c|}{9} \\
        \cline{2-13}
        & \text{|$P$|} &  \text{\#$\mathcal{V}$}  & \text{train} & \text{test}& \text{|$P$|}& \text{\#$\mathcal{V}$}& \text{train}& \text{test} & \text{|$P$|}& \text{\#$\mathcal{V}$}& \text{train}& \text{test} \\
        \hline\hline
        The {\fontsize{7}{12}\selectfont$\widetilde{\mathcal{A}}$} \normalsize{function}  & 751 & 1 & 79.50 & 80.51 & 745 & 1 & 86.01 & 85.11 & 712 & 1 & 77.54 & 80.24\\ 
        $\textsc{Coarsen}$ & 480 & 751 & 98.68& 98.78 & 491 & 745 & 98.90 & 98.59 & 506 & 712 & 98.51 & 97.45 \\ 
        $\textsc{Adversarial\_Prune}$ & 618 & 15 & 83.13 & 71.32 & 630 & 18 & 79.02 & 76.51 & 699 & 21 & 82.66 & 84.12 \\ 
        % $\textsc{Gradient\_Search}$ & 195 & - & 92.1 & 91.80 & 237 & - & 89.73 & 88.39 & 169& - & 87.59 & 88.51 \\ 
        % $\textsc{Sample\_Refine}$ & 564 & 4059 & 90.03 & 87.91 & 570 & 3663 & 88.36 & 86.84 & 581 & 4532 & 82.90 & 83.46 \\ 
        $\textsc{Sample\_Coarsen}$ & 532 & 33 & 93.19 & 93.01 & 559 & 27 & 94.12 & 93.68 & 562 & 25 & 93.89 & 93.64 \\ 
        \hline
    \end{tabular}}
\end{table}

\section{Learned Minimal NAP Specifications on ImageNet with Deep Convolutional Neural Network (Complement)}
\label{appendix:ImageNet}

\begin{table}
\centering
\caption{Overview of the size of learned minimal NAP specifications on the ImageNet benchmark. }

\label{tab:vgg_NAP}
\resizebox{\textwidth}{!}{
\begin{tabular}{|l|cc|cc|cc|cc|cc|} 
\hline
& \multicolumn{2}{c|}{\text{box\_turtle}} & \multicolumn{2}{c|}{\text{labrador\_retriever}} & \multicolumn{2}{c|}{\text{acorn\_squash}} & \multicolumn{2}{c|}{\text{confectionery}} & \multicolumn{2}{c|}{stone\_wall}  \\ 
\cline{2-11}
&\text{|$P$|} &\textit{test} &\text{|$P$|}  &\textit{test} &\text{|$P$|} &\textit{test}&\text{|$P$|}  &\textit{test} &\text{|$P$|} &\textit{test} \\
\hline\hline
The {\fontsize{7}{12}\selectfont$\widetilde{\mathcal{A}}$} \normalsize{function} & 1978 & 39.42           & 691 & 39.61           & 1003 & 29.23                  & 878 & 33.84                  & 971 & 31.54                   \\
$\textsc{Adversarial\_Prune}$                                                                                                & 1863 & 39.42           & 661  & 41.28           & 823 & 25.68                  & 845 & 23.07                  & 865 & 39.61                   \\
$\textsc{Gradient\_Search}$                                                                                                  & 611  & 53.30           & 572  & 87.01           & 301 & 52.40                  & 256 & 44.23                  & 260 & 54.90                   \\
\hline
\end{tabular}}
\end{table}

\section{The Refine Approaches}
\subsection{The Refine Approach}
Conceptually, the \textsc{Refine} approach iteratively increases the number of refined neurons in NAP $P$ until $\mathcal{V}(P) = 1$, i.e., $P$ is able to prove the underlying robustness query. In other words, we gradually increase the size parameter $k$ and iterate over each NAP $P$ of size $k$ to check if $\mathcal{V}(P) = 1$, as illustrated in Algorithm \ref{alg:Refine}. To determine if a solution to the problem exists, we first check if the most refined NAP can succeed in verification. We proceed to iterative refinement only if $\mathcal{V}(\widetilde{P}) = 1$. However, the algorithm is not efficient and requires $2^{|N|}-1$ calls to $\mathcal{V}$ in the worst case,  as proven in Theorem \ref{thm:theorem_simplerefine}.  Please refer to the proof in Appendix \ref{appendix:proof_simple}. Therefore \textsc{Refine} is only practical when the search space of the NAP family $\mathcal{P}$ is small.

\begin{theorem}
\label{thm:theorem_simplerefine}
The algorithm \textsc{Refine} returns a minimal NAP specification with $\mathcal{O}(2^{|N|})$ calls to $\mathcal{V}$.
\end{theorem}

\begin{algorithm2e}[H]
    \caption{\textsc{refine}}
    \label{alg:Refine}
    \small
    \DontPrintSemicolon
    
    \SetKwProg{Fn}{Function}{}{end}
    \SetKwFunction{Coarsen}{Refine}
    \KwIn{The neural network $N$ }
    \KwOut{A minimal NAP specification $P$}
    
    % \Fn{Pick\_K\_Neurons($k$, $N$)}{
    %     $......$ \tcc*{This function returns a list of all $k$ neuron tuples out of the neural network $N$}
    %     \Return $K\_tups$
    % }

    \Fn{Refine(N)}{
    $P \leftarrow  \widetilde{\mathcal{A}}(N)$ \\
    \uIf{$\mathcal{V} (P) == 0$}{
        \Return $None$ ; \tcc*{Return None if even the most refined NAP fails verification}
    }
    \Else{
        \For{$k$ from 1 \textbf{to} $|N|$}{
            $P \leftarrow  \dot{\mathcal{A}}(N)$ \tcc*{Refinement starts from the coarsest NAP $\dot{P}$}
            $K\_comb \leftarrow$ $Pick\_K\_Neurons(k, N)$ \tcc*{Returns all combinations of size k}
            \For{$comb \textbf{ in } K\_comb$}{
                \For{$N_{i, l} \textbf{ in } comb$}{
                    $P \leftarrow \widetilde{\Delta}(N_{i, l})$\tcc*{Refine neurons in the chosen combination }
                }
                \If{$|\mathcal{V}(P)| == 1$}{
                    \Return $P$ \tcc*{Found the minimal NAP}
            }
        }
    }}}
\end{algorithm2e}

\subsection{The Statistical Refine Approach} 
essential neurons are crucial for forming NAP specifications, as their binary states play a critical role in determining the neural network's robustness performance. We leverages this property to find essential neurons statistically. To be more specific, suppose we sample some NAPs $P_1, P_2, \ldots, P_n$ from the NAP family $\mathcal{P}$. For those NAPs that qualify as specifications (i.e., $\mathcal{V}(P) = 1$), essential neurons should appear more frequently in them than in those NAPs that fail the verification tool (i.e., $\mathcal{V}(P) = 0$). 

Based on this insight, we propose an approach called \textsc{Sample\_Refine} that relies on non-repetitive sampling to identify essential neurons for solving the minimal NAP problem. We start with the coarsest NAP and iteratively collect the most probable essential neurons. In every iteration, we sample $k$ NAPs by refining unvisited neurons with some probability $\theta$. The sampled NAPs are then fed to the verification tool, and the neuron that appears most frequently in verifiable NAPs is the most probable essential neuron in this iteration. The neuron is marked as visited, and the process stops when either we collect $s$ neurons (assuming that $s$ is known), or the current essential neurons form a NAP specification $P$, i.e., $\mathcal{V}(P) = 1$. Finally, we return the learned NAP, obtained by applying $\widetilde{\mathcal{A}}$ to the collected neurons. The algorithm~\ref{alg:samplerefine} provides an overview of the above procedure.

It is worth noting that \textsc{Sample\_Refine} doesn't guarantee correctness, as we may end up collecting only the $s$ most probable essential neurons, which may not be sufficient to form a specification. Another concern regarding this algorithm is sampling efficiency, specifically the potential for the number of samples required to grow exponentially with the size of the minimal NAP specification $s$. To understand why this is problematic, consider a scenario where the only minimal NAP specification $P$ consists of all essential neurons; in this case, all $|M|$ neurons must be selected for $P$ to be learned. If $\theta$ is set to a constant value, then the expected number of samples needed to obtain the NAP specification is $(\frac{1}{\theta})^{|M|}$. To address this, we set $\theta = \left(\frac{|M|}{|M|+1}\right)^{|M|}$. This choice ensures that the sampling efficiency is polynomial in both $|M|$ and $s$, as proven in Theorem \ref{thm:theorem_refine}. Please refer to the proof in the Appendix \ref{appendix:proof_stats}.  In addition, Theorem \ref{thm:theorem_refine} also shows that with high probability, a essential neuron will be found with $O(log|N|)$ calls to $\mathcal{V}$.

\begin{algorithm2e}[H]
    \caption{\textsc{Sample\_Refine}}
    \label{alg:samplerefine}
    \small
    \DontPrintSemicolon
    \SetKwProg{Fn}{Function}{}{end}
    \SetKwFunction{Coarsen}{Coarsen}
    \KwIn{The neural network $N$, the probability $\theta$, sample size $k$, and the size $s$}
    \KwOut{A minimal NAP specification $P$}
    
    \Fn{Sample\_NAPs($unvisited$, $\theta$)}{
        $P \leftarrow \dot{\mathcal{A}}(N)$ \tcc*{Use the coarsest NAP as a blank template}
        \For{$N_{i, l}$ \textbf{in} $unvisited$}{
            $rand \leftarrow \text{random}(0,1)$ \;
            \uIf{$rand \leq \theta$}{
                $P \leftarrow \widetilde{\Delta}(N_{i, l})$ \tcc*{Refine unvisited neurons using $\widetilde{\Delta}$ with probability $\theta$}
            }
        }
        \Return $P$\;
    }

    \Fn{Sample\_Refine($visited$, $\theta$, $k$, $s$)}{
    
    \While{$|visited| < s \text{ and }  \mathcal{V}(P) == 0 $ }{
        $unvisited \leftarrow N \setminus visited$ \;
        $ctr = dict()$ \tcc*{Create a counter for each neuron in $unvisited$} 
        \For{$\_$ \textbf{in} $range(k)$}  {
            $P \leftarrow Sample\_NAPs(unvisited, \theta)$ \tcc*{Sample $k$ NAPs} 
            \For{$N_{i, l}$ \textbf{in} $unvisited$}{
                \uIf{$P_{i,l} == \widetilde{\mathcal{A}}(N_{i, l})$ \textbf{and} $\mathcal{V}(P) == 1$}{
                    $ctr[N_{i, l}]$ += 1 \tcc*{Count successful NAPs for each neuron}
                }
            }
        }
        $N_{mad} \leftarrow \argmax_{N_{i,l} \in unvisited} \{ctr[N_{i, l}]\}$ \tcc*{Pick potential essential neuron} 
        $visited \leftarrow visited \cup \{N_{mad}\}$ \;
    }
    \Return $\widetilde{\mathcal{A}}(visited)$\;
    }
    
    $P \leftarrow \widetilde{\mathcal{A}}(N)$\;
    \uIf{$\mathcal{V}(P) == 0$}{
        \Return $None$ \tcc*{Return None if the most refined NAP fails verification}
    }
    \Else{
        \uIf{heuristics}{
            $visited \leftarrow Gradient\_Search(N) \cap Adversarial\_Prune(N) $ 
        }
        \Else{
            $visited \leftarrow \emptyset$ \tcc*{Start from the coarsest NAP}
        }
        $Sample\_Refine(visited, \theta, k, s)$\;
    }
    
\end{algorithm2e}

\begin{theorem}
\label{thm:theorem_refine}
With probability $\theta = |(\frac{|M|}{|M|+1})|^{|M|} $, \textsc{Sample\_Refine} has \( 1 - \delta \) probability of outputting a minimal NAP specification with $\Theta\left(|M|^2(\log|N| + \log(s/\delta))\right)$ examples each iteration and \( O(s|M|^2(\log|N| + \log(s/\delta))) \) total calls to \( \mathcal{V} \).
\end{theorem}

\section{Proofs of simple approaches}
\label{appendix:proof_simple}

\begin{theorem} [Property of \textsc{Refine}]
The algorithm \textsc{Refine} returns a minimal NAP specification with $\mathcal{O}(2^{|N|})$ calls to $\mathcal{V}$.
\end{theorem}

\begin{proof}
Let $P$ be the returned NAP. We prove this by contradiction. Suppose we can further refine $P$, meaning there exists a NAP $P'$ such that $|P'| \leq |P|$ and $\mathcal{V}(P') = 1$. However, the algorithm states that any $P'$ smaller than $|P|$ fails verification, which contradicts $\mathcal{V}(P') = 1$.

In the worst case, the NAP size $k$ runs up to $|N|$. For each $k$, we need to check $\binom{|N|}{k}$ number of NAPs. In total, this number of NAPs we need to check is $\binom{|N|}{1} + \binom{|N|}{2} + \cdots + \binom{|N|}{|N|} = 2^{|N|} - 1$ according to the binomial theorem, resulting in a runtime complexity of $\mathcal{O}(2^{|N|})$.

\end{proof}

\begin{theorem} [Property of \textsc{Coarsen}]
The algorithm \textsc{Coarsen} returns a minimal NAP specification with $\mathcal{O}(|N|)$ calls to $\mathcal{V}$.
\end{theorem}

\begin{proof}
Let $P$ be the NAP returned by \textsc{Coarsen}. Our goal is to show that any $P'$ smaller than $P$ results in $\mathcal{V}(P') = 0$. To construct such a smaller $P'$, we need to apply the refine action $\widetilde{\Delta}$ on $P$ through some neuron $N_{i, l}$, i.e., $P' := \widetilde{\Delta}(N_{i, l}) = \widetilde{\mathcal{A}}(N_{i, l}) \cup P \setminus P_{i,l}$. According to the algorithm, $\mathcal{V}(P') = 0$. In the worst case, the algorithm needs to iterate through each neuron in $N$, resulting in a runtime complexity of $\mathcal{O}(|N|)$.
\end{proof}

\section{Proofs of statistical approaches}
\label{appendix:proof_stats}

Our proofs of properties of \textsc{Sample\_Refine} and \textsc{Sample\_Coarsen} mainly follow those in \cite{minimal_abs}. Interested readers may refer to it for detailed proofs.

\begin{theorem} [Property of \textsc{Sample\_Refine}]
With probability $\theta = |(\frac{|M|}{|M|+1})|^{|M|} $, \textsc{Sample\_Refine} has \( 1 - \delta \) probability of outputting a minimal NAP specification with $\Theta\left(|M|^2(\log|N| + \log(s/\delta))\right)$ examples each iteration and \( O(s|M|^2(\log|N| + \log(s/\delta))) \) total calls to \( \mathcal{V} \).
\end{theorem}

\begin{proof}
Considering the size of the minimal specification $s$, \textsc{Sample\_Refine} will execute $s$ iterations. If we sample $k$ times in each iteration, then the probability of selecting a essential neuron is at least $1 - \frac{\delta}{s}$. Consequently, by applying a union bound, the algorithm will identify a NAP specification with a probability of at least $1 - \delta$.

Now, let's delve deeper into one iteration of the process. The fundamental concept is that a essential neuron $m$ demonstrates a stronger correlation with proving the robustness query ($\mathcal{V}(P)=1$) compared to a non-essential one. This enhanced correlation increases the probability of its selection significantly when $k$ is sufficiently large.

Let's revisit the notion of $M$, representing the set of essential neurons with a size of $|M|$. Now, let's focus on a specific essential neuron $n^{+} \in M$. We define $B_{j-}$ as the event indicating $k_{n^{-}} > k_{n^{+}}$, and $B$ as the event where $B_{j-}$ holds for any non-essential neuron $n^{-} \in N \setminus M$. Importantly, if $B$ fails to occur, then the algorithm will correctly identify a essential neuron. Hence, our primary objective is to establish that $Pr(B) \leq \frac{\delta}{s}$. Initially, employing a union bound provides:
\begin{align}
    Pr(B) \leq \sum_{n^{-}} Pr(B_{n^{-}}) \leq |N| \max_{n^{-}}Pr(B_{n^{-}})
\end{align}
Let's delve into each training example $P^{(i)}$ and introduce the notation $X_i = (1 - \mathcal{V}(P^{(i)}))(P_{n^{-}}^{(i)} - P_{n^{+}}^{(i)})$. It's worth emphasizing that $B_{n^{-}}$ manifests precisely when $\frac{1}{n}(k_{n^{-}} - k_{n^{+}}) = \frac{1}{n}\sum_{i = 1}^n X_i > 0$. Our objective now is to bound this quantity utilizing Hoeffding's inequality, considering the mean as:
\begin{align}
    \mathbb{E}[X_i] = Pr(\mathcal{V}(P)=1, P_{n^{-}} \in \{\mathbf{1},\mathbf{0}\}) - Pr(\mathcal{V}(P)=1, P_{n^{+}} \in \{\mathbf{1},\mathbf{0}\})
\end{align}
and the bounds are $-1 \leq \mathbb{E}[X_i] \leq 1$ . Setting $\epsilon = -\mathbb{E}[X_i]$, we get:
\begin{align}
    Pr(B_{n^{-}}) \leq e^{\frac{-k\epsilon^2}{2}}, n^{+} \in M, n^{-} \notin M.
\end{align}
Substituting (16) into (15) and rearranging terms, we can solve for $k$:
\begin{align}
\frac{\delta}{s} \leq |N|e^{\frac{-k\epsilon^2}{2}} \text{ implies } k \geq \frac{2(log|N| + log(\frac{s}{\delta}))}{\epsilon^2}
\end{align}
Our attention now shifts towards deriving a lower bound for $\epsilon$, which intuitively reflects the discrepancy (in terms of correlation with proving the robustness query) between a essential neuron and a non-essential one. It's noteworthy that $Pr(P_n \in {\{\mathbf{1},\mathbf{0}}\}) = \theta$ for any $n \in N$. Furthermore, since $n^{-}$ is non-essential, we can assert that $Pr(\mathcal{V}(P)=1|P_{n^{-}} \in \{\mathbf{1},\mathbf{0}\}) = Pr(\mathcal{V}(P)=1)$. By leveraging these observations, we can express:
\begin{align}
    \epsilon = \theta(Pr(\mathcal{V}(P)=1 |P_{n^{+}} \in \{\mathbf{1},\mathbf{0}\}) - Pr(\mathcal{V}(P)=1))
\end{align}
Let's view $C$ as the collection of minimal NAP specifications. We can treat $C$ as a set of clauses in a Disjunctive Normal Form (DNF) formula: $\mathcal{V}(P; C) = \neg \bigvee_{c \in C}\bigwedge_{n \in c} P_n$, where we explicitly specify the dependence of $\mathcal{V}$ on the clauses $C$. For example, if $C = {{1, 2}, {3}}$, it corresponds to $\mathcal{V}(P) = \neg [(P_1 \land P_2)\lor P_3]$. Now, let $C_j = {c \in C: n \in c}$ denote the clauses containing $n$. We reformulate $Pr(\mathcal{V}(P)=1)$ as the sum of two components: one originating from the essential neuron $n^{+}$ and the other from the non-essential neuron:
\begin{align}
    Pr(\mathcal{V}(P)=1) &= Pr(\mathcal{V}(P; C_{n^{+}}) = 1, \mathcal{V}(P; C \setminus C_{n^{+}}) = 0) \\
    &\quad+ Pr(\mathcal{V}(P; C \setminus C_{n^{+}}) = 1)
\end{align}
Calculating $Pr(\mathcal{V}(P)=1 |P_{n^{+}} \in \{\mathbf{1},\mathbf{0}\})$ follows a similar process. The only distinction arises from conditioning on $P_{n^{+}} \in \{\mathbf{1},\mathbf{0}\}$, which introduces an extra factor of $\frac{1}{\theta}$ in the first term because conditioning divides by $Pr(P_{n^{+}}) = \theta$. The second term remains unchanged since no $c \notin C_{n^{+}}$ is essential in $P_{n^{+}}$. Substituting these two outcomes back into the equation yields:
\begin{align}
    \epsilon = (1 - \theta)Pr(\mathcal{V}(P; C_{n^{+}}) = 1, \mathcal{V}(P; C \setminus C_{n^{+}}) = 0)
\end{align}
Now, our objective is to establish a lower bound for (22) across all possible $\mathcal{V}$ (equivalently, $C$), where $n^{+}$ is permitted to be essential in $C$. Interestingly, the worst possible $C$ can be obtained by either having $|M|$ disjoint clauses ($C = {{n}: n \in M}$) or a single clause ($C = {M}$ if $s = |M|$). The intuition behind this is that when $C$ consists of $|M|$ clauses, there are numerous possibilities ($|M| - 1$ of them) for some $c \notin C_{n^{+}}$ to be true, making it challenging to determine that $n^{+}$ is a essential neuron; in such cases, $\epsilon = (1 - \theta)\theta(1 -\theta)^{|M| - 1}$. Conversely, if $C$ comprises a single clause, then letting this clause be true becomes exceedingly challenging; in this scenario, $\epsilon = (1 - \theta)\theta^{|M|}$.

Let's consider the scenario where $C$ comprises $|M|$ clauses. We aim to maximize $\epsilon$ concerning $\theta$ by setting the derivative $\frac{d\epsilon}{d\theta} = 0$ and solving for $\theta$. This optimization yields $\theta = \frac{1}{|M| + 1}$ as the optimal value. Substituting this value into the formula of $\epsilon$, we obtain $\epsilon = \frac{1}{|M| + 1}(\frac{|M|}{|M| + 1})^{|M|}$. Note that 
$(\frac{|M|}{|M| + 1})^{|M|}$ can be lower bounded by $e^{-1}$, implying that $\epsilon^{-2} = \mathcal{O}(|M|^2)$. Substituting this result to equation (18) will conclude the proof.
\end{proof}

\begin{theorem} [Property of \textsc{Stochastic\_Coarsen}]
With probability $\theta$=$e^{-\frac{1}{s}}$, \textsc{Sample\_Coarsen} learns a minimal NAP specification with $\mathcal{O}(slog|N|)$ calls to $\mathcal{V}$.
\end{theorem}

\begin{proof}
Let's first estimate the number of calls that \textsc{Sample\_Coarsen} makes to $\mathcal{V}$. We denote the number of calls as $\mathcal{C}(P^L)$, where $P^L$ is the most refined NAP. Then $\mathcal{C}(P^L)$ can be computed recursively using the following rule:
\begin{align}
    \mathcal{C}(P^L) = 
    \begin{cases}
    |P| & \text{ if } |P^L| \leq s + 1 \\
    1 + \mathbb{E}[(1 - \mathcal{V}(P))\mathcal{C}(P) + \mathcal{V}(P)\mathcal{C}(P^L)] & \text{ otherwise }\\
  \end{cases}
\end{align}
where $P$ is the sampled NAP. By assumption, there exists a NAP $P^S \preccurlyeq P$ of size $s$ that passes the verification. Define $G(P) = \neg (P^S \preccurlyeq P)$, which is 0 when $P^S$ subsumes $P$, i.e., all essential neurons in $P^S$ shows up in the sampled $P$. This follows that $Pr(G(P) = 0) = Pr(P^S \preccurlyeq P) = \theta^s$. Note that $G(P) \geq \mathcal{V}(P)$, as the NAP $P^S$ suffices to prove the robustness query. We can estimate the upper bound of $\mathcal{C}(P^L)$ by replacing $\mathcal{V}$ with $G$:
\begin{align}
    \mathcal{C}(P^L) &\leq 1 + \mathbb{E}[(1 - G(P))\mathcal{C}(P) + G(P)\mathcal{C}(P^L)] \\
    &\leq 1 + \theta^s\mathbb{E}[\mathcal{C}(P) | P^S \preccurlyeq P] + (1 - \theta^s)\mathcal{C}(P^L) \\
    &\leq \mathbb{E}[\mathcal{C}(P) | P^S \preccurlyeq P] + \theta^{-s}
\end{align}

We now denote $\mathcal{C}(n) = max_{|P| = n} \mathcal{C}(P)$ as the maximum over NAP of size $n$. Note that given $P^S \preccurlyeq P$, $|P| = s + N$ where $N$ is a binomial random variable with $\mathbb{E}(N) = \theta(n - s)$.

Using the  bound $\mathcal{C}(n) \leq (1 - \theta^n)\mathcal{C}(n - 1) + \theta^n\mathcal{C}(n)\theta^{-s}$, we can observe that $\mathcal{C}(n) \leq \frac{\theta^{-s}}{1 - \theta^n} \cdot n$. In addition, when $n$ is large enough, $\mathcal{C}(n)$ is concave, then by  use Jensen's inequality:
\begin{align}
    \mathcal{C}(n) \leq \mathcal{C}(\mathbb{E}[s + N]) + \theta^{-s}  = \mathcal{C}(s + \theta(n - s)) + \theta^{-s}
\end{align}
To solve the recurrence, this gives us:
\begin{align}
    \mathcal{C}(n) \leq \frac{\theta^{-s}log n}{log \theta^{-1}} + s + 1
\end{align}

The equation above illustrates a tradeoff between reducing the number of iterations (by increasing $\log \theta^{-1}$) and reducing the number of samples (by decreasing $\theta^{-s}$). To minimize $\mathcal{C}(n)$, we need to set $\theta$ so that the gradient of $\mathcal{C}(n)$ w.r.t $\theta^{-1}$ is $0$. This gives us: $\frac{sx^{s - 1}}{\log x} - \frac{x^{s - 1}}{\log^2 x} = 0$, solving this gives us $\theta = e^{-\frac{1}{s}}$. Consequently, the upper bound becomes $\mathcal{C}(n) = es\log n + s + 1 = O(s\log n)$.
\end{proof}